\documentclass[sigconf]{acmart}

\AtBeginDocument{%
  }

\makeatletter
\def\@ACM@copyright@check@cc{}
\makeatother

\copyrightyear{2025}
\acmYear{2025}
\setcopyright{cc}
\setcctype{by}
\acmConference[ICCPS '25]{ACM/IEEE 16th International Conference on Cyber-Physical Systems (with CPS-IoT Week 2025)}{May 6--9, 2025}{Irvine, CA, USA}
\acmBooktitle{ACM/IEEE 16th International Conference on Cyber-Physical Systems (with CPS-IoT Week 2025) (ICCPS '25), May 6--9, 2025, Irvine, CA, USA}
\acmDOI{10.1145/3716550.3722037}
\acmISBN{979-8-4007-1498-6/2025/05}

\usepackage{xcolor}

\newtheorem{assumption}{Assumption}
\newtheorem{lemma}{Lemma}
\newtheorem{example}{Example}
\newtheorem{theorem}{Theorem}
\usepackage{graphicx} 

\usepackage{changes}
\usepackage{multirow}
\usepackage[inline]{enumitem}
\usepackage{float}
\usepackage{hyperref}
\usepackage{listings}

\begin{document}

\title[SEVIN]{Scalable and Interpretable Verification of Image-based Neural Network Controllers for Autonomous Vehicles}

%
\author{Aditya Parameshwaran}
\email{aparame@clemson.edu}
\orcid{0009-0004-9688-4580}
\affiliation{%
  \institution{Clemson University}
  \city{Clemson}
  \state{South Carolina}
  \country{USA}
}

\author{Yue Wang}
\email{yue6@clemson.edu}
\orcid{0000-0003-0146-7262}
\affiliation{%
  \institution{Clemson University}
  \city{Clemson}
  \state{South Carolina}
  \country{USA}
}

\renewcommand{\shortauthors}{Parameshwaran et al.}

\begin{abstract}
Existing formal verification methods for image-based neural network controllers in autonomous vehicles often struggle with high-dimensional inputs, computational inefficiency, and a lack of interpretability. These challenges make it difficult to ensure safety and reliability, as processing high-dimensional image data is computationally intensive and neural networks are typically treated as black boxes. To address these issues, we propose \textbf{SEVIN} (Scalable and Interpretable Verification of Image-Based Neural Network Controllers), a framework that leverages Variational Autoencoders (VAE) to encode high-dimensional images into a lower-dimensional, interpretable latent space. By annotating latent variables with corresponding control actions, we generate convex polytopes that serve as structured input spaces for verification, significantly reducing computational complexity and enhancing scalability. Integrating the VAE's decoder with the neural network controller allows for formal and robustness verification using these interpretable polytopes. Our approach also incorporates robustness verification under real-world perturbations by augmenting the dataset and retraining the VAE to capture environmental variations. Experimental results demonstrate that SEVIN achieves efficient and scalable verification while providing interpretable insights into controller behavior, bridging the gap between formal verification techniques and practical applications in safety-critical systems.
\end{abstract}



\received{31 October 2024}

\maketitle

\section{Introduction}

Ensuring the safety and reliability of image-based neural network controllers in autonomous vehicles (AVs) is paramount. These controllers process high-dimensional inputs, such as images from front cameras, to make real-time control decisions. However, existing formal verification methods~\cite{Katz2017Reluplex:Networks, Gehr2018AI2:Networks, Singh2019AnAbstraction} face significant challenges due to the high dimensionality and complexity of image inputs, leading to computational inefficiency and scalability issues~\cite{bunel2018unified, singh2018fast}. Moreover, these methods often treat neural networks as black boxes, offering limited interpretability and making it difficult to understand how specific inputs influence outputs—an essential aspect for safety-critical applications like AVs.

Recent efforts have employed abstraction-based methods~\cite{Gehr2018AI2:Networks, Singh2019AnAbstraction} and reachability analysis~\cite{ruan2018reachability} to approximate neural network behaviors. Specification languages grounded in temporal logic~\cite{pnueli1977temporal, vasilache2022verifying} and Satisfiability Modulo Theory (SMT) solvers~\cite{ehlers2017formal} have been used to formalize and verify properties. Several innovative approaches have emerged to specifically address the verification challenges of image-based neural networks in autonomous systems. Tran et al. introduced ImageStars, a set representation technology that efficiently handles the high dimensionality of image inputs while providing formal guarantees~\cite{Tran2020Verification}. In parallel, Katz et al. leveraged generative models such as General Adversarial Networks (GAN) to create a lower-dimensional latent space for closed loop verification, thus reducing computational complexity~\cite{Katz2021Verification}. While their goal is to conduct closed loop verification along with a linearized plant model, our aim is to evaluate the performance of different neural network controllers with input spaces generated by both clean and augmented image datasets. For autonomous systems specifically,  Julian et al. developed an adaptive stress testing framework that identifies critical scenarios where neural networks controllers might fail~\cite{Julian2020Validation}, while Al-Nuaimi et al. proposed hybrid verification techniques that combine formal methods with simulation-based testing to achieve more comprehensive safety guarantees~\cite{Al-Nuaimi2021Hybrid} . Despite these advances, creating scalable and interpretable verification methods for image-based neural network controllers remains an open challenge. Furthermore, robustness verification under real-world input perturbations also remains unsolved. Modeling and analyzing variations efficiently is difficult due to the complexity of image data and environmental factors affecting AVs. Consequently, current approaches lack methods that:
\begin{itemize} 
    \item \textbf{Reduce Computational Complexity:} Effectively handle the high dimensionality of image inputs without compromising verification thoroughness. 
    \item \textbf{Enhance Interpretability:} Provide insights into how input features influence control actions, facilitating better understanding and trust. 
    \item \textbf{Improve Scalability:} Scale to larger datasets and more complex controllers, especially when considering robustness against real-world perturbations. 
\end{itemize}

To address these limitations, we propose SEVIN (\emph{Scalable and Interpretable Verification of Image-Based Neural Network Controllers}), a novel approach that leverages unsupervised learning with a Variational Autoencoder (VAE)~\cite{kingma2013autoencoding} to learn a structured latent representation of the controller's input space. By encoding high-dimensional image data into a lower-dimensional, interpretable latent space, we significantly reduce the computational complexity of the verification process, making it more scalable. We define \textbf{interpretable} formal verification as the process of not only mathematically proving that a neural network satisfies certain properties but also providing a human-understandable correlation between symbolic properties, input space and output space.

Our method involves training a VAE on a dataset of image-action pairs collected from a driving simulator. The latent space is partitioned into convex polytopes corresponding to different control actions, enabling us to define formal specifications over these polytopes. By operating in this latent space, we enhance interpretability and gain insights into how latent features influence control actions.

We further extend our approach to incorporate robustness verification under input perturbations common in real-world scenarios for AVs. By augmenting the dataset with perturbed images and retraining the VAE, we ensure that the latent space captures variations due to environmental changes, sensor noise, and other factors affecting image inputs.

Our experimental results demonstrate that SEVIN not only achieves efficient and scalable verification of image-based neural network controllers but also provides interpretable insights into the controller's behavior. This advancement bridges the gap between formal verification techniques and practical applications in safety-critical systems like AVs. In summary, we make the following contributions

\subsection{Summary of Contributions}
\begin{itemize} 
    \item[(1)] An interpretable latent space is developed for a neural network controller dataset by employing a Gaussian Mixture-VAE model. The encoded variables are annotated according to the control actions correlated with their high-dimensional inputs, enabling the derivation of convex polytopes as defined input spaces for the verification process.
    
    \item[(2)] A streamlined and scalable framework is then constructed to integrate the VAE’s decoder network with the neural network controller, facilitating formal and robustness verification of the controller by utilizing the interpretable convex polytopes as structured input spaces.
    
    \item[(3)] Finally, symbolic specifications are synthesized to encapsulate the safety and performance properties of two image-based neural network controllers. Using these specifications in conjunction with the \texttt{$\alpha-\beta-CROWN$} neural network verification tool \cite{zhang2018efficient, xu2020automatic, xu2021fast, wang2021beta}, formal and robustness verification of the controllers is effectively performed.

\end{itemize}

\section{Preliminaries}
\subsection{Variational Autoencoder (VAE)}\label{sec:VAE}

VAEs are generative models that compress input data (\( \mathbf{x} \)) into a latent space and then reconstruct the input (\( \mathbf{\hat{x}} \)) from the compressed latent representation \cite{kingma2013autoencoding,pmlr-v32-rezende14}. A VAE $V(\mathbf{x})$, consists of an encoder \( E(\mathbf{x}) \) and a decoder \( D(\mathbf{z}) \), where \( \mathbf{z} \) is the latent variable capturing the compressed representation of the input data.

In variational inference, the true posterior distribution \( p(\mathbf{z}|\mathbf{x}) \) is often intractable to compute directly and hence an approximate posterior \( q(\mathbf{z}|\mathbf{x}) \) is introduced \cite{kingma2013autoencoding}. The encoder maps the input data to a latent distribution \( q_{\phi}(\mathbf{z}|\mathbf{x}) \), parameterized by \( \phi \), while the decoder reconstructs the input data from the latent variable using \( p_{\theta}(\mathbf{x}|\mathbf{z}) \), parameterized by \( \theta \). Instead of directly calculating for the intractable marginal likelihood \( p(\mathbf{x}) \), VAEs maximize the Evidence Lower Bound (ELBO) to provide a tractable lower bound to \( \log p(\mathbf{x}) \) \cite{kingma2013autoencoding}:
\begin{equation}\label{eq:VAE_loss}
\text{ELBO} = \underbrace{\mathbb{E}_{\mathbf{z} \sim q_{\phi}(\mathbf{z}|\mathbf{x})} [\log p_{\theta}(\mathbf{x} \ | \ \mathbf{z})]}_{\text{Reconstruction Term}} - \underbrace{\mathcal{D}_{KL}[q_{\phi}(\mathbf{z} \ |\ \mathbf{x}) \, \| \, p(\mathbf{z})]}_{\text{KL Divergence Term}}
\end{equation}
The ELBO consists of a reconstruction term that encourages the decoded output to be similar to the input data, and a Kullback-Leibler (KL) divergence term that regularizes the latent space to match a prior distribution \( p(\mathbf{z}) \). The prior \( p(\mathbf{z}) \) is often chosen as a standard Gaussian \cite{kingma2013autoencoding}, but can be more flexible, such as a Gaussian mixture model \cite{dilokthanakul2016deep} or VampPrior \cite{tomczak2018vae}, depending on the desired latent space structure.

\subsection{Neural Network Verification}\label{prelim:nnv}
Neural network verification tools are designed to rigorously analyze and prove properties of neural networks, ensuring that they meet specified input-output requirements under varying conditions \cite{liu2019algorithms}. Consider an \( L \)-layer neural network representing the function $F(\mathbf{x})$ for which the verification tools can determine the validity of the property as:
\begin{equation}\label{eq:nnv_1}
\mathbf{x} \in X  \implies F(\mathbf{x}) \in A  
\end{equation}

where \( X \) and \(A \) are the convex input and output sets, respectively \ \cite{KatzVerificationModels}. The neural network verification ensures that for all inputs in a specified set \( X \), the outputs of the neural network \( F(\mathbf{x}) \) satisfy certain properties defined by a set \( A \). The weights and biases for $F(\mathbf{x})$ are represented as \( \mathbf{W}^{(i)} \in \mathbb{R}^{d_n^{(i)} \times d_n^{(i-1)}} \) and \( \mathbf{b}^{(i)} \in \mathbb{R}^{d_n^{(i)}} \), where $d_n^{(i)}$ is the dimensionality for layer \( i \in \{1, \ldots, L\} \) for the $L$-layered neural network. The neural network function \( F(\mathbf{x}) = h^{(L)}(\mathbf{x}) \):
\begin{equation}\label{eq:nn_def}
    \begin{aligned}
        h^{(i)}(\mathbf{x}) &= \mathbf{W}^{(i)} \hat{h}^{(i-1)}(\mathbf{x}) + \mathbf{b}^{(i)}, \\
        \hat{h}^{(i)}(\mathbf{x}) &= \sigma\left( h^{(i)}(\mathbf{x}) \right), \\
        \hat{h}^{(0)}(\mathbf{x}) &= \mathbf{x}
    \end{aligned}
\end{equation}

where \( \sigma \) denotes the activation function. When the ReLU activation function is used, the neural network verification problem \eqref{eq:nnv_1} becomes a constrained optimization problem with the objective function as shown below~\cite{tjeng2019evaluating}, 
\begin{equation}\label{eq:nnv_optimization}
    \begin{aligned}
        &F_{\min} = \min_{\mathbf{x} \in X} F(\mathbf{x}), \quad  F_{\max} = \max_{\mathbf{x} \in X} F(\mathbf{x}) && \\
        \text{s.t.} \quad & F_{\min}, 
        F_{\max}\in A &&
    \end{aligned}
\end{equation}
where \( F(\mathbf{x}) \) is defined as the set of piecewise-linear functions from Equation~\eqref{eq:nn_def}. Since the ReLU activation functions are piecewise linear, allowing the neural network to be represented as a combination of linear functions over different regions of the input space, the verification problem can be formulated as an optimization problem solvable by techniques such as Mixed-Integer Linear Programming (MILP) \cite{tjeng2019evaluating} and SMT solvers~\cite{kaiser2021smt}.

Furthermore, \textit{robust formal} verification is a specific aspect of neural network verification that focuses on the network's resilience to small perturbations in the input data~\cite{huang2017safety}.
The robustness verification problem can be formalized as:
\begin{displaymath}\label{eq:robustness_problem}
 \mathbf{x} \in B(\mathbf{x}_0, \delta) \implies F(\mathbf{x}) \in A
\end{displaymath}

where \( B(\mathbf{x}_0, \delta) = \{\mathbf{x} | \|\mathbf{x} - \mathbf{x}_0\| \leq \delta\} \) represents a norm-bounded perturbation around a nominal input \( \mathbf{x}_0 \), and \( \delta >0 \) is the perturbation limit.
Alternatively, robustness verification can also be formulated as a constrained optimization problem:
\begin{equation}\label{eq:robustness_optimization}
    \begin{aligned}
        & F_{\min} = \min_{\mathbf{x} \in B(\mathbf{x}_0, \delta)} F(\mathbf{x}), \quad F_{\max}= \max_{\mathbf{x} \in B(\mathbf{x}_0, \delta)} F(\mathbf{x})\\
        \text{s.t.} \quad & F_{\min}, 
        F_{\max}\in A
    \end{aligned}
\end{equation}

By solving the optimization problems in \eqref{eq:nnv_optimization} and \eqref{eq:robustness_optimization} within their defined input sets, and verifying that the corresponding outputs reside within the target set \( A \), verification tools such as Reluplex \cite{Katz2017Reluplex:Networks} and AI$^2$ \cite{Gehr2018AI2:Networks} provide essential guarantees for vanilla formal and robustness verification. 

\begin{figure*}[!h] 
    \centering
    \includegraphics[width=1.0\textwidth]{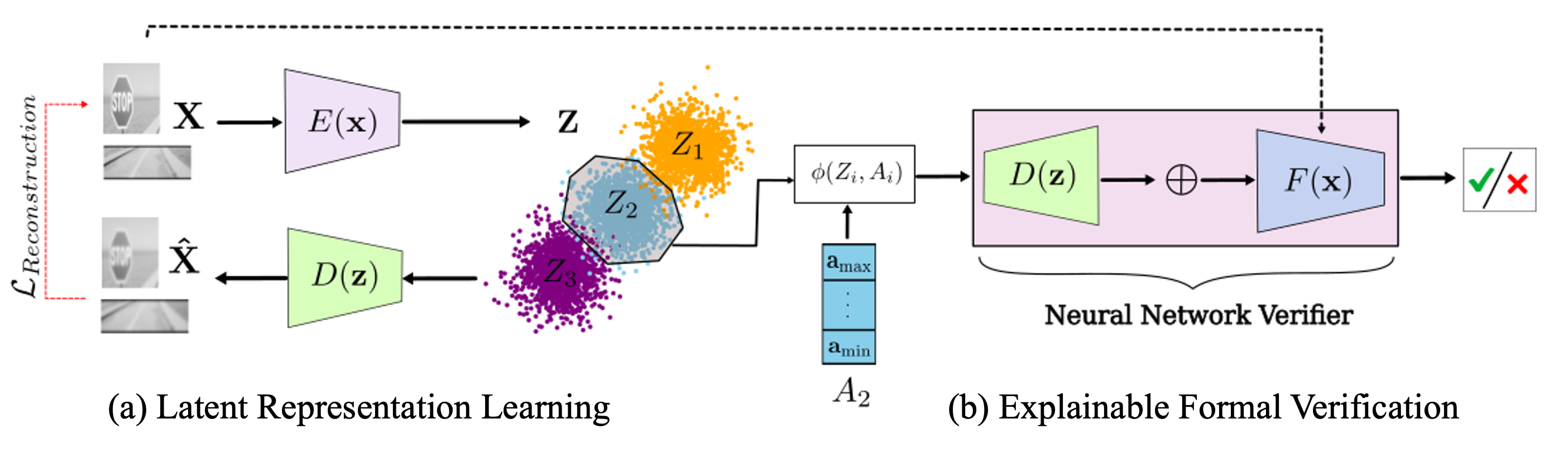}
    \caption{The SEVIN model can be decomposed into two sub-modules for conducting formal verification of any neural network controller. (a) A VAE, defined as $\hat{\mathbf{x}} = D(E(\mathbf{x}))$, is initially trained and utilized to learn representation sets ($Z_i$) of latent features from the dataset $X$. The same dataset $X$ is also used to train the image-based neural network controller $F(\mathbf{x})$, which will later undergo verification. (b) Any latent feature sample $\mathbf{z} \in Z_i$ is representative of the dominant features that influence the control action ($\mathbf{a}$) predicted by $F(\mathbf{x})$. By combining the decoder $D(\mathbf{z})$ with the neural network controller $F(\mathbf{x})$, we can determine a set of control actions $A_i$ based on 
     $Z_i$ (see more in Lemma~\ref{lemma_1}). Finally, using a neural network verification tool, we can formally verify the satisfaction of the neural network controller $F(\mathbf{x})$ against a formal specification ($\phi$)}.
    \label{fig:approach}
\end{figure*}

\subsection{Symbolic Specification Language}\label{sec:nsl}
The symbolic specification language defines properties for neural network verification, integrating principles from \emph{Linear Temporal Logic} (LTL)~\cite{pnueli1977temporal} to express dynamic, time-dependent behaviors essential for cyber-physical systems like AV.

LTL formulas are defined recursively as:
\begin{displaymath}\label{eq:ltl_syntax}
\Phi ::= \text{true} | a | \varphi_1 \lor \varphi_2 | \neg \varphi | \bigcirc \varphi | \varphi_1 \, \mathcal{U} \, \varphi_2
\end{displaymath}

where
\begin{itemize}
    \item \( \text{true} \) denotes the Boolean constant \texttt{True}.
    \item \( a \) is an atomic proposition, typically about network inputs or outputs.
    \item \( \lor \), \( \neg \), \( \bigcirc \), and \( \mathcal{U} \) represent disjunction, negation, next, and until operators.
\end{itemize}

Using the above LTL formulas, other operators like ``always" (\( \square \varphi \)) and ``eventually" (\( \lozenge \varphi \)) can be defined $\square \varphi \equiv \neg \lozenge \neg \varphi, \lozenge \varphi \equiv \text{true} \, \mathcal{U} \, \varphi.$

These temporal operators allow precise specification of properties over time, such as safety (\( \square \varphi \)) and liveness (\( \lozenge \varphi \)) conditions. Specification methods based on LTL~\cite{pnueli1977temporal,vasilache2022verifying}, Signal Temporal Logic (STL)~\cite{akazaki2018falsification}, Satisfiability Modulo Theories (SMT)~\cite{ehlers2017formal}, and other formal techniques provide the basis for rigorous neural network verification, enabling precise and reliable analysis of temporal behaviors in dynamic environments.

\begin{example}
    Consider an image based neural network controller $F(\mathbf{x})$ that is trained to predict steering action values ($\mathbf{a}$). We expect that for all images in the subset of left-turn images $X_{\text{left}}$, the controller should predict negative action values corresponding to turning left, i.e.,
    \begin{displaymath}
    \mathbf{x} \in X_{\text{left}} \implies F(\mathbf{x}) \in A_{\text{left}}
    \end{displaymath}

    Formal verification of the neural network $F(\mathbf{x})$ thus corresponds to solving the following optimization problem:
    \begin{displaymath}\label{eq:example_1_op}
        \begin{aligned}
            &F_{\min} = \min_{\mathbf{x} \in X_{\text{left}}} F(\mathbf{x}), F_{\max} = \max_{\mathbf{x} \in X_{\text{left}}} F(\mathbf{x}) \\
            \text{s.t.} \quad & F_{\min} , F_{\max}\in A_{\text{left}}
        \end{aligned}
    \end{displaymath}
    
    We can thus formally define the input specification to the neural network verification tool using the symbolic specification language operators as:
    \begin{displaymath}\label{eq:example_1_prop}
        \varphi := \square (\ \{F_{\min},F_{\max}\} \in A_{\text{left}})
    \end{displaymath}
    
    If the verification tool can show that the specification $\varphi$ stands \text{true} for the given input set $X_{\text{left}}$, then the specification is satisfied (\textbf{SAT}), or else the specification is unsatisfied (\textbf{UNSAT}).
\end{example}

\section{Our Solution}\label{problem_form}

We begin by collecting images and control actions from a driving simulator, forming a dataset of image-action pairs $(\mathbf{x}, \mathbf{a})$, where $X = \{\mathbf{x}_j\}_{j=1}^N$ consists of $N$ front camera images and $A = \{\mathbf{a}_j\}_{j=1}^N$ consists of the corresponding control actions. A VAE $V(\mathbf{x})$ is trained to learn the latent representation $Z$ of this dataset, encoding high-dimensional image data into a lower-dimensional latent space (see Section~\ref{sec:latent_space}). This encoding reduces the computational complexity of the verification problem, making it more scalable.

Once the VAE is trained, we label the latent variables ($\mathbf{z}$) based on their corresponding control actions ($\mathbf{a}$). This labeling allows us to partition the latent space $Z$ into convex polytopes $C_i$, such that $C = \bigcup_{i \in I} C_i$, as elaborated in Section~\ref{sec:conv_poly}. Each polytope $C_i$ corresponds to a specific control action set $A_i$, with the action space expressed as $A = \bigcup_{i \in I} A_i$. This partitioning generates an interpretable input space for the formal verification process, enhancing the understanding of how inputs influence outputs.

We then split the trained VAE into encoder $E(\mathbf{x})$ and decoder $D(\mathbf{z})$ networks and concatenate the decoder with the controller network $F(\mathbf{x})$. The combined network $\mathcal{H}(\mathbf{z}) = F(D(\mathbf{z}))$ maps variables directly from the latent space to control actions. By operating in the latent space instead of the high-dimensional image space, we significantly reduce the input dimensionality and computational complexity of the verification problem, making the process more scalable and computationally efficient.

Formal specifications $\varphi$ are defined using the symbolic specification language described in Section~\ref{sec:nsl}, capturing the safety and performance properties of the neural network controller $F(\mathbf{x})$. The combined network $\mathcal{H}(\mathbf{z})$ and the specifications $\varphi$ are provided to a neural network verification tool, such as $\alpha$-$\beta$-CROWN, which uses bound propagation and linear relaxation techniques to certify properties of neural networks. The verification tool checks whether $\mathcal{H}(\mathbf{z})$ satisfies the specifications $\varphi$ over the input convex polytopes in the latent space. The equivalence between verifying $F(\mathbf{x})$ and $\mathcal{H}(\mathbf{z})$ is established in Theorem~\ref{theorem_1}. 

To address robustness verification under input perturbations common in real-world scenarios, we extend our approach by training the VAE on a dataset of both clean and augmented images (see Section~\ref{sec:aug_latent_space}). The augmented dataset $\bar{X}$ is generated by applying quantifiable perturbations—such as changes in brightness, rotations, translations, and motion blurring—to the original images (see Figure~\ref{fig:img_augmentations}). The corresponding latent representation set $\bar{Z}$ is used by SEVIN to generate augmented latent space convex polytopes $\bar{C}_i$ for robustness analysis. The overall formal verification process aims to assess the neural network controller's performance under two distinct conditions:

\textit{Vanilla formal verification}: For a clean input space $C_i$, drawn from the subset $X_i \subseteq X$, assumed to consist of unperturbed, front-camera-captured images.  

\textit{Robust formal verification}: For an augmented input space $\bar{C}_i$ from the subset $\bar{X_i} \subseteq \bar{X}$, where $\bar{X}$ comprises images with applied, quantifiable augmentations relevant to AV scenarios.

This extension makes the verification process more scalable by incorporating robustness verification into the same framework without significant additional computational complexity.

\section{Scalable and interpretable Verification of Image-based Neural Networks (SEVIN)}

\subsection{Latent Representation Learning of the Dataset}\label{sec:latent_space}

To develop a scalable and interpretable framework for neural network controller verification, we first train a VAE on the dataset of images used in the neural network controller training process. The VAE is tasked with reconstructing front-camera images captured by an AV while learning latent representations that capture the underlying structure and variability in the data—such as different driving conditions, environments, and vehicle behaviors—in a compressed and informative form.

 Let \( V(\mathbf{x}): \mathbb{R}^{h \times w} \rightarrow \mathbb{R}^{h \times w} \) represent a Gaussian Mixture Variational Autoencoder (GM-VAE) trained over a dataset of images \( X = \{\mathbf{x}_i\}_{i=1}^M \subset \mathbb{R}^{h \times w} \) to learn a structured latent space representation \( \mathbf{z} \in Z \subset \mathbb{R}^{d_z} \) and reconstruct images \( \hat{\mathbf{x}} = V(\mathbf{x}) \in \hat{X} \subset \mathbb{R}^{h \times w} \). Here, \( h \times w \) denotes the height and width of the images. We assume a Gaussian mixture prior over the latent variables \( \mathbf{z} \), defined as:
\begin{equation}\label{eq:GM-prior}
    p(\mathbf{z}) = \sum_{k=1}^K s_k \, \mathcal{N}(\mathbf{z} | \boldsymbol{\mu}_k, \boldsymbol{\Sigma}_k)
\end{equation}

where \( \boldsymbol{\mu}_k \) and \( \boldsymbol{\Sigma}_k \) represent the mean and covariance matrix of the \( k \)-th Gaussian component in the latent space, and \( s_k \) represents the mixture weight for each Gaussian, satisfying \( \sum_{k=1}^K s_k = 1 \). The VAE \( V(\mathbf{x}) \) comprises an encoder model \( E(\mathbf{x}): \mathbb{R}^{h \times w} \rightarrow \mathbb{R}^{3 \times K \times d_z} \) and a decoder model \( D(\mathbf{z}): \mathbb{R}^{d_z} \rightarrow \mathbb{R}^{h \times w} \). The encoder \( E(\mathbf{x}) \) outputs a \( \{3 \times K \times d_z\} \) dimensional array, where \( [\boldsymbol{\mu}_k, \log \boldsymbol{\Sigma}_k, s_k] \) corresponds to the \( k \)-th Gaussian in every latent dimension. The variable \( d_z \) denotes the dimensionality of the latent space.
An advantage of using a GM-VAE is that it provides a more flexible latent space representation compared to a standard VAE, especially when the data exhibit multiple modes~\cite{dilokthanakul2016deep}.

To train the GM-VAE loss function ($\mathcal{L}$), we employ the loss function defined in \cite{dilokthanakul2016deep} as:
\begin{displaymath}\label{eq:gm_vae_loss}
    \mathcal{L} = 
     \ \operatorname{MSE}\left(\mathbf{x}, \hat{\mathbf{x}}\right) \nonumber
     + \beta \sum_{k=1}^{K} \, 
    \mathcal{D}_{KL} \left[ q(\mathbf{z} | \mathbf{x}, k) 
    \,\|\, p(\mathbf{z}, k) \right]
\end{displaymath}

The Mean Squared Error (MSE) between the input data and the reconstructed images is used to calculate the reconstruction loss described in \eqref{eq:VAE_loss}. Maximizing the likelihood $p(\mathbf{x}|\mathbf{z})$ under the Gaussian mixture is equivalent to minimizing the MSE between \( \mathbf{x} \) and \( \hat{\mathbf{x}} \), as the negative log-likelihood of a Gaussian distribution with fixed variance simplifies to MSE loss \cite{kingma2013autoencoding}. Here, \( q_{\phi}(\mathbf{z}|\mathbf{x}, k) \) is the posterior probability of selecting the \( k \)-th component of the mixture for input \( \mathbf{x} \), and the KL divergence term \( \mathcal{D}_{KL} \) measures the discrepancy between the posterior and the corresponding prior Gaussian component $p (\mathbf{z}, k)$. The hyper-parameter \( \beta \in \mathbb{R}^{+} \) balances the trade-off between reconstructing the input data accurately and minimizing the divergence between the approximate posterior and the prior distribution, similar to the concept introduced in the $\beta$-VAE framework~\cite{higgins2017betaVAE}. This formulation allows the VAE to learn more complex latent structures by capturing multi-modal distributions in the latent space \cite{dilokthanakul2016deep, kingma2013autoencoding}.

\begin{assumption}\label{assumption_1}
Given a GM-VAE  \( V(\mathbf{x}) \) trained on a dataset of images \( X \), we assume that the reconstructed images \( \hat{\mathbf{x}} = V(\mathbf{x}) \) satisfy \( \| F(\hat{\mathbf{x}}) - F(\mathbf{x}) \|_2 \leq \epsilon \), where \( \epsilon > 0 \).
\end{assumption}

For a given neural network controller \( F(\mathbf{x}) \), the parameter \( \epsilon \) is a measurable quantity that depends solely on the training efficacy of \( V(\mathbf{x}) \). Our objective is to optimize \( V(\mathbf{x}) \) to achieve \( \epsilon \approx \frac{1}{100} \| F(\mathbf{x}) \|_2 \), which is considered an acceptable error threshold in our AV driving scenarios.

\subsection{Interpretable Latent Space Encoding}\label{interpret_latent}
The encoder \( E(\mathbf{x}) \) learns a mapping from high-dimensional images \( \mathbf{x} \) to low-dimensional latent representations \( \mathbf{z} \). Images with similar features—such as lane markings, traffic signs, or attributes like brightness and blur—are mapped close together in the latent space because the encoder learns to associate these common attributes with nearby regions. The continuity of the latent space enforced by the KL divergence ensures that similar inputs have similar latent representations. The decoder \( D(\mathbf{z}) \) regenerates the front-camera images \( \hat{\mathbf{x}} \) from the latent variables \( \mathbf{z} \in Z \).

In our approach, each image \( \mathbf{x} \) is associated with a corresponding control action \( \mathbf{a} \), such as steering angle or linear velocity, which acts as a label for the image for training purposes. The control action \( \mathbf{a} \) is the action to be predicted by the neural network controller \( F(\mathbf{x}) \) based on the image \( \mathbf{x} \). The dataset \( X \) collected from driving simulations is randomly split 70/30 for training and validating the VAE \( V(\mathbf{x}) \) and the neural network controller \( F(\mathbf{x}) \) respectively. The datasets are generated and labeled automatically during simulation, where the vehicle's control actions are recorded alongside the images captured. Further details on the data collection process are provided in Section~\ref{sec:driving_scenarios}.

\begin{lemma}\label{lemma_1}
Let \( X = \{\mathbf{x}_i\}_{i=1}^M \) be a dataset of images, and let \( A \subset \mathbb{R}^m \) be a set of action values, where each image \( \mathbf{x}_i \) is associated with a control action \( \mathbf{a}_i \in A \) that the neural network controller \( F(\mathbf{x}) \) should predict.

Using the encoder \( E(\mathbf{x}) \), the images are mapped to latent variables \( \mathbf{z}_i = E(\mathbf{x}_i) \), resulting in the set of latent variables \( Z = \{\mathbf{z}_i\}_{i=1}^M \). The latent variables are assumed to follow the GM prior distribution described in \eqref{eq:GM-prior}. Then, for each action value \( \mathbf{a} \in A \), the set of latent variables corresponding to \( \mathbf{a} \) has positive probability under \( p(\mathbf{z}) \). Specifically, the probability of sampling a latent variable \( \mathbf{z} \) such that there exists a pair \( (\mathbf{z}, \mathbf{a}) \) in \( Z \times A \) is greater than zero:
\begin{equation}\label{eq:lemma1_result}
    \forall \mathbf{a} \in A, \quad p\left( \exists \mathbf{z} \in Z \ \text{s.t.} \ (\mathbf{z}, \mathbf{a}) \in Z \times A \right) > 0
\end{equation}
\end{lemma}

\begin{proof}\label{sec:proof_1}
For each action value \( \mathbf{a} \in A \), there exists at least one image \( \mathbf{x} \in X \) such that the associated control action is \( \mathbf{a} \). By applying the encoder to any such image \( \mathbf{x} \), we obtain the latent variable \( \mathbf{z} = E(\mathbf{x}) \). Thus, the pair \( (\mathbf{z}, \mathbf{a}) \) exists in \( Z \times A \).

Since the latent variables \( \mathbf{z} \) are generated from images \( \mathbf{x} \) via the encoder \( E(\mathbf{x}) \), and the latent space \( Z \) follows the GM prior \( p(\mathbf{z}) \), this prior \( p(\mathbf{z}) \) serves as the probability density function over \( Z \). Therefore, \( p(\mathbf{z}) \) assigns a positive probability to all latent variables in \( Z \) that correspond to images in \( X \) under the mapping \( E(\mathbf{x}) \).

Formally, for any \( \mathbf{a} \in A \), we define the set of latent variables corresponding to \( \mathbf{a} \) as:

\begin{equation}
    Z_{\mathbf{a}} = \left\{ \mathbf{z} \in Z | \exists \mathbf{x} \in X \ \text{s.t.} \ \mathbf{a} = F(\mathbf{x}) \text{ and } \mathbf{z} = E(\mathbf{x}) \right\}
\end{equation}

Since \( Z_{\mathbf{a}} \) is non-empty and \( p(\mathbf{z}) \) is a valid probability density function over \( Z \), it follows that:

\begin{equation}
    p\left( \mathbf{z} \in Z_{\mathbf{a}} \right) = \int_{Z_{\mathbf{a}}} p(\mathbf{z}) \, d\mathbf{z} > 0
\end{equation}

Thus, there exists a positive probability of sampling a latent variable \( \mathbf{z} \) corresponding to any action \( \mathbf{a} \in A \), establishing the result stated in Equation~\eqref{eq:lemma1_result}.
\end{proof}

The aim was to show that there is a non-zero probability of sampling a corresponding $\mathbf{z}$ such that the pair ($\mathbf{z,a}$) exists in the latent space. This allows us to sample variables in the latent space with the corresponding action value acting as their labels. Figure \ref{fig:latent_space} illustrates the 8-dimensional latent space of the image dataset $X$, where each point is labeled according to its corresponding action value set $A_i$. The latent space has been reduced to 2 dimensions using t-SNE (t-distributed Stochastic Neighbor Embedding) for visualization. t-SNE is a dimensionality reduction technique that transforms high-dimensional data into a lower-dimensional space while preserving the structure of data clusters, making it ideal for visualization~\cite{vanDerMaaten2008}. Due to the clustering nature of the VAE, the latent variables with similar features are clustered and can be interpreted by their action values.

\begin{figure}[h]
    \centering
    \includegraphics[width=\columnwidth]{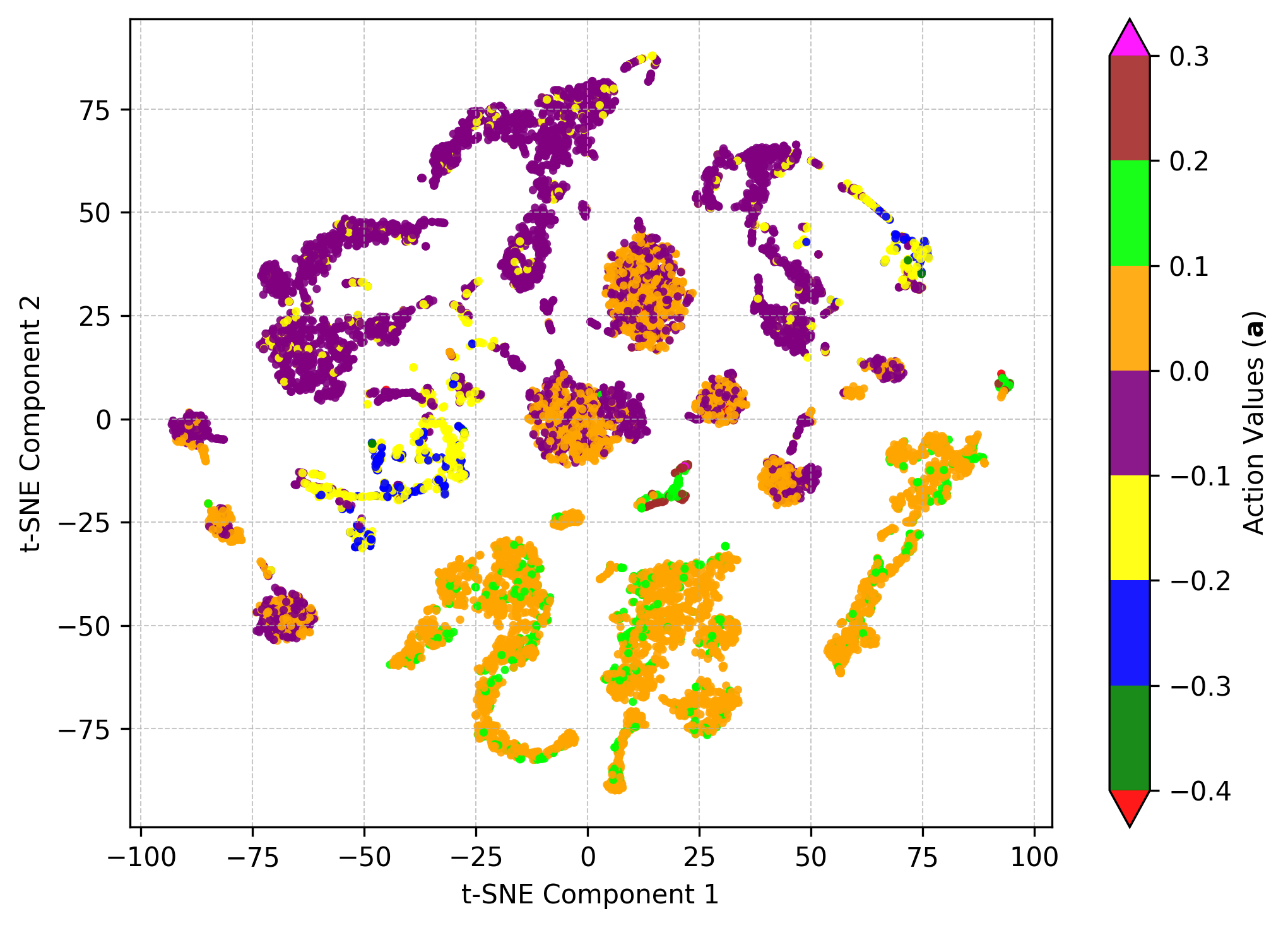}
    \caption{t-SNE plot of an 8D latent space representation generated for the clean image dataset ($X$)}
    \label{fig:latent_space}
\end{figure}

\subsection{Latent Space Convex Polytope Formulation}\label{sec:conv_poly}

For a collection of continuous action subsets \( \{ A_i \}_{i \in I} \), where \( I \) is an index set and \( A = \bigcup_{i \in I} A_i \) covers the entire action space, we construct corresponding convex polytopes \( C_i \subset \mathbb{R}^{d_z} \) in the latent space. These polytopes approximate the regions associated with each action subset \( A_i \). To generate \( C_i \) for a given action subset \( A_i \), we perform Monte Carlo sampling of the latent variables \( \mathbf{z} \) from the GM prior \( p(\mathbf{z}) \), focusing on samples corresponding to \( A_i \). Specifically, for each \( A_i \), we consider the set of images \( X_i \subseteq X \) such that each image \( \mathbf{x} \in X_i \) is associated with an action \( \mathbf{a} \in A_i \), i.e., 
\begin{displaymath}
X_i = \left\{ \mathbf{x} \in X \mid F(\mathbf{x}) \in A_i \right\}
\end{displaymath}

Using the encoder \( E(\mathbf{x}) \), we map the images \( \mathbf{x} \in X_i \) to their latent representations \( \mathbf{z}\), resulting in latent variables \( \mathbf{z} \in Z_i \), where \( Z_i \) denotes the set corresponding to \( A_i \). This establishes the correspondence between the image-action pairs \( (\mathbf{x}, \mathbf{a}) \) and the latent-action pairs \( (\mathbf{z}, \mathbf{a}) \).
To construct \( C_i \), we draw \( n \) independent samples of the latent variable \( \mathbf{z} \) from the GM prior $p(\mathbf{z})$, ensuring that each sample belongs to \( Z_i \) corresponding to \( A_i \) and is within \( 2 \) standard deviations from the mean \( \boldsymbol{\mu}_{A_i} \) of $p(\mathbf{z} | \mathbf{z} \in Z_i)$. The convex polytope \( C_i \) is then defined as the convex hull of these sampled latent variables:
\begin{equation}\label{eq:conv_polytope}
C_i = \operatorname{conv}\left( \left\{ \mathbf{z}_j | \mathbf{z}_j = E(\mathbf{x}_j), \ \mathbf{x}_j \in X_i, \ j = 1, \dots, n \right\} \right)
\end{equation}

where \( \operatorname{conv}(\cdot) \) denotes the operation of the convex hull~\cite{grunbaum2003convex}. Constructing \( C_i \) in this manner provides an under-approximation of the latent space region corresponding to \( A_i \), as it is based on finite samples that are only within 2 standard deviations from the mean ($\mu_{A_i}$). Increasing \( n \) improves the approximation and coverage. For thorough formal verification, it is often preferred to include potential variations and edge cases in the process as well. Hence, we enlarge the convex polytope \( C_i \) uniformly by applying a Minkowski sum~\cite{schneider1993convex} with a ball \( R(0, \epsilon) \) centered at the origin with radius \( \epsilon > 0 \):
\begin{equation}\label{eq:minkowski}
\tilde{C_i} = C_i \oplus R(0, \epsilon)
\end{equation}

where \( \oplus \) denotes the Minkowski sum and the ball \( R(0, \epsilon) \) is defined as:
\begin{displaymath}
R(0, \epsilon) = \left\{ \mathbf{r} \in \mathbb{R}^{d_z} | \ \| \mathbf{r} \|_2 \leq \epsilon \right\}
\end{displaymath}

The Minkowski sum expands \( C_i \) by \( \epsilon \) in all directions, resulting in an enlarged polytope \( \tilde{C_i} \)~\cite{schneider1993convex}. This enlargement accounts for slight variations or noise while maintaining the association with \( A_i \). To compute \( \tilde{C_i} \) explicitly, we can represent it as:
\begin{displaymath}\label{eq:minkowski_polytope}
\tilde{C_i} = \operatorname{conv} \left( \bigcup_{j=1}^n \left\{ \mathbf{z}_j \oplus \mathbf{r} \mid \| \mathbf{r} \|_2 \leq \epsilon \right\} \right)
\end{displaymath}

 We assume that $C_i$ and $\tilde{C_i}$ both contain latent variables corresponding to the action set $A_i$. We use the Quickhull algorithm~\cite{Barber1996Quickhull} to identify the vertices of \( C_i \) and \( \tilde{C_i} \), facilitating their construction and use in verification tasks.

\begin{lemma}\label{lemma_2}
    The convex polytope \( C_i \) defines a continuous input space, and any latent variable \( \mathbf{z} \in C_i \) decoded using the decoder \( D(\mathbf{z}) \) generates a high-dimensional reconstructed image, \( \hat{\mathbf{x}} \in \mathbb{R}^{h \times w} \), that forms a pair \( (\hat{\mathbf{x}}, \mathbf{a})^i \), where \( \mathbf{a} \in A_i \). This relationship can be expressed as:
    \begin{displaymath}\label{eq:image_action_pair}
    \forall \mathbf{z} \in C_i, \quad (\hat{\mathbf{x}}, \mathbf{a})^i = \left( D(\mathbf{z}), \mathbf{a} \right), \quad \mathbf{a} \in A_i
    \end{displaymath}

\end{lemma}

\begin{proof}\label{sec:proof_2}
    By construction, the convex polytope \( C_i \subset \mathbb{R}^{d_z} \) is formed from \( n \) latent samples \( \{\mathbf{z}_j\}_{j=1}^n \) drawn from the GM prior \( p(\mathbf{z}) \) within two standard deviations from the mean \( \boldsymbol{\mu}_{A_i} \) corresponding to the action set \( A_i \). Specifically, each sampled latent variable \( \mathbf{z}_j \) satisfies:
    
    \begin{displaymath}
        \| \mathbf{z}_j - \boldsymbol{\mu}_{A_i} \|_2 \leq 2\sigma_{A_i}
    \end{displaymath}
    
     Since \( C_i \) is the convex hull of these samples, any \( \mathbf{z} \in C_i \) can be expressed as a linear combination of the sampled latent variables:
     
    \begin{displaymath}
        \mathbf{z} = \sum_{j=1}^n \lambda_j \mathbf{z}_j, \quad \text{where } \lambda_j \geq 0 \text{ and } \sum_{j=1}^n \lambda_j = 1
    \end{displaymath}
    
    The decoder \( D(\mathbf{z}) \) is assumed to be a continuous function mapping latent variables to high-dimensional images. Therefore, decoding \( \mathbf{z} \in C_i \) yields:
    
    \begin{displaymath}
        \hat{\mathbf{x}} = D(\mathbf{z}) = D\left( \sum_{j=1}^n \lambda_j \mathbf{z}_j \right)
    \end{displaymath}
    
    Specifically, since each latent variable \( \mathbf{z}_j \) corresponds to an action \( \mathbf{a}_j \in A_i \), and $A_i$ is a continuous action set, the convex combination ensures that \( \hat{\mathbf{x}} \) is associated with an action \( \mathbf{a} \in A_i \).
    
    Formally, for each \( \mathbf{z} \in C_i \), there exists an image \( \hat{\mathbf{x}} \in \mathbb{R}^{h \times w} \) such that:
    
    \begin{displaymath}
        (\hat{\mathbf{x}}, \mathbf{a})^i \in \mathbb{R}^{h \times w} \times A_i
    \end{displaymath}
    
    This establishes that decoding any latent variable within the convex polytope \( C_i \) produces an image associated with the action set \( A_i \).
    
    Consequently, the convex polytope \( C_i \) effectively encapsulates a continuous region in the latent space corresponding to the action set \( A_i \), ensuring that all decoded images \( \hat{\mathbf{x}} \) from \( C_i \) are correctly paired with actions \( \mathbf{a} \in A_i \).
\end{proof}

\begin{example}\label{example2}
    Consider the latent space illustrated in Figure~\ref{fig:convex_polytope}. A GM-VAE \(V(\mathbf{x})\) was trained to learn and decode the latent representation \(Z\) for a dataset that contains pairs of front-camera images and control steering actions, represented by \((\mathbf{x}, \mathbf{a})\). The correlation between \((\mathbf{x}, \mathbf{a})\) is nonlinearly mapped by the encoder \(E(\mathbf{x})\) to the latent space. The two colors illustrated in the latent space correspond to discrete sets of control actions \(\{A_1 , A_2\}\). Our objective is to construct a convex polytope \(C_1^{\epsilon}\), represented by the light-shaded region, such that it contains all latent variables \(\mathbf{z}\) labeled with action values within the range \(A_1 = [0.02, 0.2]\). After training the encoder \(E(\mathbf{x})\), we sample the latent variable set $Z_1$ such that \(Z_1 = \{\mathbf{z}_j\}_{j=1}^{1000}\) through Monte Carlo sampling from the GM prior shown in Equation \eqref{eq:GM-prior}. Once $Z_1$ is obtained, the Quickhull algorithm is applied to determine the boundary vertices for \(C_1\) using the convex polytope formulation depicted in Section \ref{sec:conv_poly}. These vertices are uniformly extended outward using the Minkowski sum, as expressed in Equation~\eqref{eq:minkowski}, with \(\epsilon = 0.05\) to handle potential edge cases. The extended polytope \(C_1^{\epsilon}\) serves as the input space for further discussions in Section~\ref{sec:formal_verification_convex}.
\end{example}

\begin{figure}[h]
    \centering
    \includegraphics[width=\columnwidth]{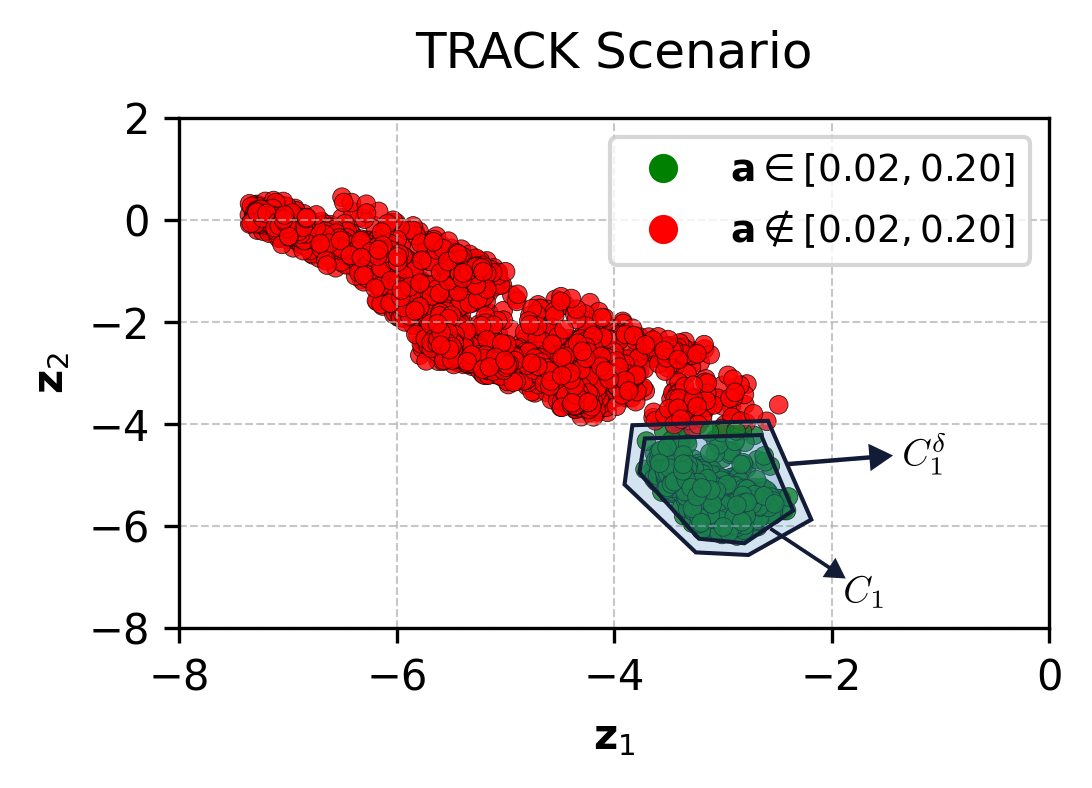}
    \caption{2D latent space representation of the front camera images from the driving scenario described in Example~\ref{example2}. The variables are labeled by action sets $A_1 = [0.02, 0.20]$ and $A_2 = A \setminus A_1$.}\Description{A 2D latent space plot showing regions labeled by action sets $A_1$ and $A_2$, derived from front camera images in a driving scenario.}
    \label{fig:convex_polytope}
\end{figure}

\subsection{Formal Verification using Latent Space Convex Polytopes}\label{sec:formal_verification_convex}
The convex polytopes \( C_i \) defined in Section~\ref{sec:conv_poly} facilitates the development of an interpretable and scalable framework for verifying image-based neural network controllers. Consider a neural network controller \( F(\mathbf{x}) \) trained on the same dataset \( X \) as the VAE from Section \ref{sec:latent_space}. Traditional formal verification processes for image-based neural networks involve solving optimization problems as expressed in Equation~\eqref{eq:nnv_optimization}. However, the input space for such problems becomes significantly large, contingent on the dimensionality of the input image \( \mathbf{x} \in \mathbb{R}^{h \times w} \). Moreover, the input space subset $X_i$ used for verification is limited to a finite and discrete selection of images, which constrains both interpretability and scalability.

In contrast, our approach leverages the convex polytope \( C_i \) to define an interpretable and continuous input space in the latent domain. Specifically, \( C_i \) encompasses a continuous region of latent variables \( \mathbf{z}_i \in C_i \), where each \( \mathbf{z}_i \) is associated with a control action \( \mathbf{a}_i \in A_i \). This transformation reduces the verification problem's input space from the high-dimensional \( \mathbb{R}^{h \times w} \) to the lower-dimensional latent space \( \mathbb{R}^{d_z} \).

To perform formal verification of the neural network controller \( F(\mathbf{x}) \), we first integrate the decoder \( D(\mathbf{z}) \) with \( F(\mathbf{x}) \), forming the combined network \( \mathcal{H} \) defined as:
\begin{displaymath}
    \mathcal{H}(\mathbf{z}) = F(D(\mathbf{z}))
\end{displaymath}

For any latent variable \( \mathbf{z} \), the combined network \( \mathcal{H}(\mathbf{z}) \) predicts a control action \( \mathbf{a} \in A\). Given that both \( F(\mathbf{x}) \) and \( D(\mathbf{z}) \) are neural networks with ReLU activation functions, the composite function \( \mathcal{H}(\mathbf{z}) \) is piecewise linear and continuous.

\begin{theorem}\label{theorem_1}
 Given \( \mathcal{H}(\mathbf{z}) = F(D(\mathbf{z})) \), where both \( F \) and \( D \) are neural networks employing ReLU activations, and \( C_i \) is a convex polytope in \( \mathbb{R}^{d_z} \) defined in Equation~\eqref{eq:conv_polytope}, finding the local minimum of \( F(\mathbf{x}) \) over \( \mathbf{x} \in D(\mathbf{z}) \) is equivalent to finding a local minimum of \( \mathcal{H}(\mathbf{z}) \) over \( \mathbf{z} \in C_i \). Formally,
    \begin{displaymath}
        \min_{\mathbf{x} \in D(C_i)} F(\mathbf{x}) \ \equiv \ \min_{\mathbf{z} \in C_i} \mathcal{H}(\mathbf{z})
    \end{displaymath}
    
\end{theorem}

\begin{proof}\label{sec:proof_3}
By \textbf{Lemma~\ref{lemma_1}}, for each action \( \mathbf{a} \in A \), there exists a latent variable \( \mathbf{z} \in Z \) such that the pair \( (\mathbf{z}, \mathbf{a}) \) is in \( Z \times A \) with positive probability under the prior \( p(\mathbf{z}) \). This ensures that the latent space \( Z \) is meaningfully connected to the action space \( A \), and the mapping between images and actions is preserved in the latent space.

According to \textbf{Lemma~\ref{lemma_2}}, any latent variable \( \mathbf{z} \in C_i \) can be decoded using \( D(\mathbf{z}) \) to generate a reconstructed image \( \hat{\mathbf{x}} \). This image forms a pair \( (\hat{\mathbf{x}}, \mathbf{a})^i \) with an action \( \mathbf{a} \in A_i \), indicating that the decoder \( D \) maps the convex polytope \( C_i \) in latent space back to meaningful images in the original space \( X \).

Neural networks with ReLU activations, such as \( F \) and \( D \), are known to be piecewise linear functions~\cite{montufar2014number, raghu2017expressive}. They partition their input spaces into polyhedral regions within which the functions act linearly. The composition \( \mathcal{H}(\mathbf{z}) = F(D(\mathbf{z})) \) thus inherits this piecewise linearity because the composition of piecewise linear functions is also piecewise linear~\cite{bronstein2017geometric}.

From the properties established in \textbf{Lemma~\ref{lemma_2}}, \( D \) maps the convex polytope \( C_i \) in latent space to a corresponding polyhedral region in the image space \( X \). This means that optimizing \( F(\mathbf{x}) \) over \( \mathbf{x} \in D(C_i) \) is equivalent to optimizing \( \mathcal{H}(\mathbf{z}) \) over \( \mathbf{z} \in C_i \), since \( D \) provides a bijective linear mapping within these regions.

Moreover, since a local minimum of a piecewise linear function occurs at a vertex or along an edge of its polyhedral regions~\cite{clarke1998nonsmooth}, the local minima of \( \mathcal{H}(\mathbf{z}) \) over \( \mathbf{z} \in C_i \) correspond directly to the local minima of \( F(\mathbf{x}) \) over \( \mathbf{x} \in D(C_i) \). Therefore, optimizing \( F(\mathbf{x}) \) over the high-dimensional space \( X_i \) is equivalent to optimizing \( \mathcal{H}(\mathbf{z}) \) over the lower-dimensional latent space polytope \( C_i \), establishing the equivalence of the two optimization problems~\cite{kawaguchi2016deep, choromanska2015loss}.

\end{proof}

Consequently, Theorem~\ref{theorem_1} demonstrates that minimizing \( F(\mathbf{x}) \) over \( \mathbf{x} \in X_i\) is equivalent to minimizing \( \mathcal{H}(\mathbf{z}) \) over \( \mathbf{z} \in C_i \). Referring to Equations~\eqref{eq:nnv_1} and ~\eqref{eq:nnv_optimization}, the verification process thus for the combined network \( \mathcal{H}(\mathbf{z}) \) can be formalized as:
\begin{displaymath}
    \mathbf{z} \in C_i  \implies \mathcal{H}(\mathbf{z}) \in A_i
\end{displaymath}

and the optimization problem to be solved to conduct the verification process can be formalized as:
\begin{equation}\label{eq:formal_obj_func}
    \begin{aligned}
        &a_{\min} = \min_{\mathbf{z} \in C_i} \mathcal{H}(\mathbf{z}), \quad && a_{\max} = \max_{\mathbf{z} \in C_i} \mathcal{H}(\mathbf{z}) \\
        \text{subject to} \quad & a_{\min} \in A_i, \quad
        &&a_{\max} \in A_i,
    \end{aligned}
\end{equation}

where \( C_i \) and \( A_i \) are correlated as described in Lemma~\ref{lemma_2}. It is important to note that the bounds of each action set \( A_i \subseteq A \) depend on the formal specifications being verified and correspond to a specific convex polytope \( C_i \subset \mathbb{R}^{d_z} \) in the \( d_z \)-dimensional latent space.

\subsection{Augmented Latent Spaces for Robustness Verification}\label{sec:aug_latent_space}

In prior sections, we developed the SEVIN framework, which leverages latent representations to formally verify neural network controllers. These representations are derived from the dataset $X$, consisting solely of unperturbed images captured by the front camera of an AV. Consequently, both the reconstructed dataset $\hat{X}$ and the set of convex polytopes $C$ correspond exclusively to clean data. In this section, we will utilize the SEVIN framework to also conduct robustness verification of the neural network controller. 

\begin{figure}[h]
    \centering
    \includegraphics[width=\columnwidth]{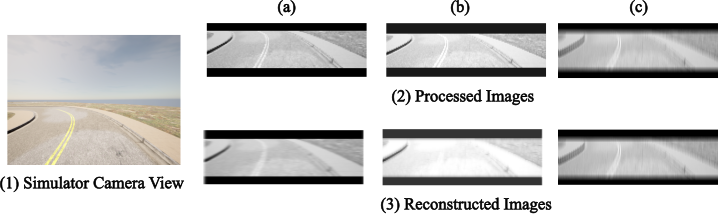}
    \caption{Illustration of the data pre-processing, augmentation, and reconstruction steps used in our approach. (1) Original front camera view captured from the AV within the \texttt{CARLA} driving simulator environment on a two-lane track. (2)(a) Resized and cropped images before training. (b) Augmented image with brightness level adjusted by a random factor $\delta \in [-0.2, 0.2]$. (c) Augmented image with motion blur applied, with the degree of blur (kernel size) $\delta \in [1,6] \cap \mathbb{Z}$. (3) Reconstructions of the images from (2) generated by the trained VAE's.}
    \label{fig:img_augmentations}
\end{figure}

Figure~\ref{fig:img_augmentations} illustrates some augmentations applied to the dataset of clean images. To verify the robustness of neural network controllers, we construct latent space representations for two datasets, each incorporating a different augmentation type: (1) Image Brightness, (2) Motion Blur, following the SEVIN approach. Each augmentation is quantified and applied to the dataset $X = \{\mathbf{x}\}_{j=1}^M$, resulting in an augmented images only dataset $\bar{X} = \bar{\{\mathbf{x}\}}_{j=1}^M$. The augmented images are then combined with the original dataset to form the new combined augmented and clean dataset, $\bar{X} = X \cup \bar{X}$. From $\bar{X}$, we generate a latent space representation $\bar{Z}$ and the set of convex polytopes $\bar{C} = \bigcup_{i=1}^I \bar{C}_i$, following the SEVIN formulation outlined in Section~\ref{sec:conv_poly}. 

Notably, for any action $\mathbf{a} \in A$, $(\bar{\mathbf{x}}, \mathbf{a}) \in \bar{X} \times A$. This is due the fact that the augmentation is applied only to the image ($\mathbf{x}$) and does not affect the control action ($\mathbf{a}$) to be taken by the controller. Importantly, the augmentations do not alter the action values for any latent variable $\mathbf{z} \in \bar{C}_i$, thereby preserving Lemma~\ref{lemma_1} such that:
\begin{displaymath}
    \forall a \in A, \quad p(\exists \ \bar{\mathbf{z}} \in \bar{Z} \text{ s.t. } (\bar{\mathbf{z}},\mathbf{a}) \in \bar{Z} \times A) \ > 0
\end{displaymath}

In fact, the VAE learns the augmentation applied to the image as an additional feature and can distinguish between the clean and augmented images quite precisely as seen in Figure~\ref{fig:aug_latent_space}.

Based on the optimization problem corresponding to robustness verification described in \eqref{eq:robustness_optimization}, we can use the SEVIN framework to formalize the robustness verification process as:
\begin{displaymath}
    \mathbf{z} \in \bar{C}_i \implies \mathcal{H}(\mathbf{z}) \in A_i
\end{displaymath}

where the optimization problem to be solved by the neural network verification tool is:
\begin{displaymath}\label{eq:robust_obj_func}
    \begin{aligned}
        &a_{\min} = \min_{\mathbf{z} \in \bar{C}_i} \mathcal{H}(\mathbf{z}), \quad && a_{\max} = \max_{\mathbf{z} \in \bar{C}_i} \mathcal{H}(\mathbf{z}) \\
        \text{subject to} \quad & a_{\min} \in A_i, \quad
        &&a_{\max} \in A_i
    \end{aligned}
\end{displaymath}

Section~\ref{sec:image_augmenations} shows the type of augmentations along with the range of values that are applied to the dataset $X$. Individual VAE's are trained for each augmented dataset $\bar{X}$.

\begin{figure}[h]
    \centering
    \includegraphics[width=\columnwidth]{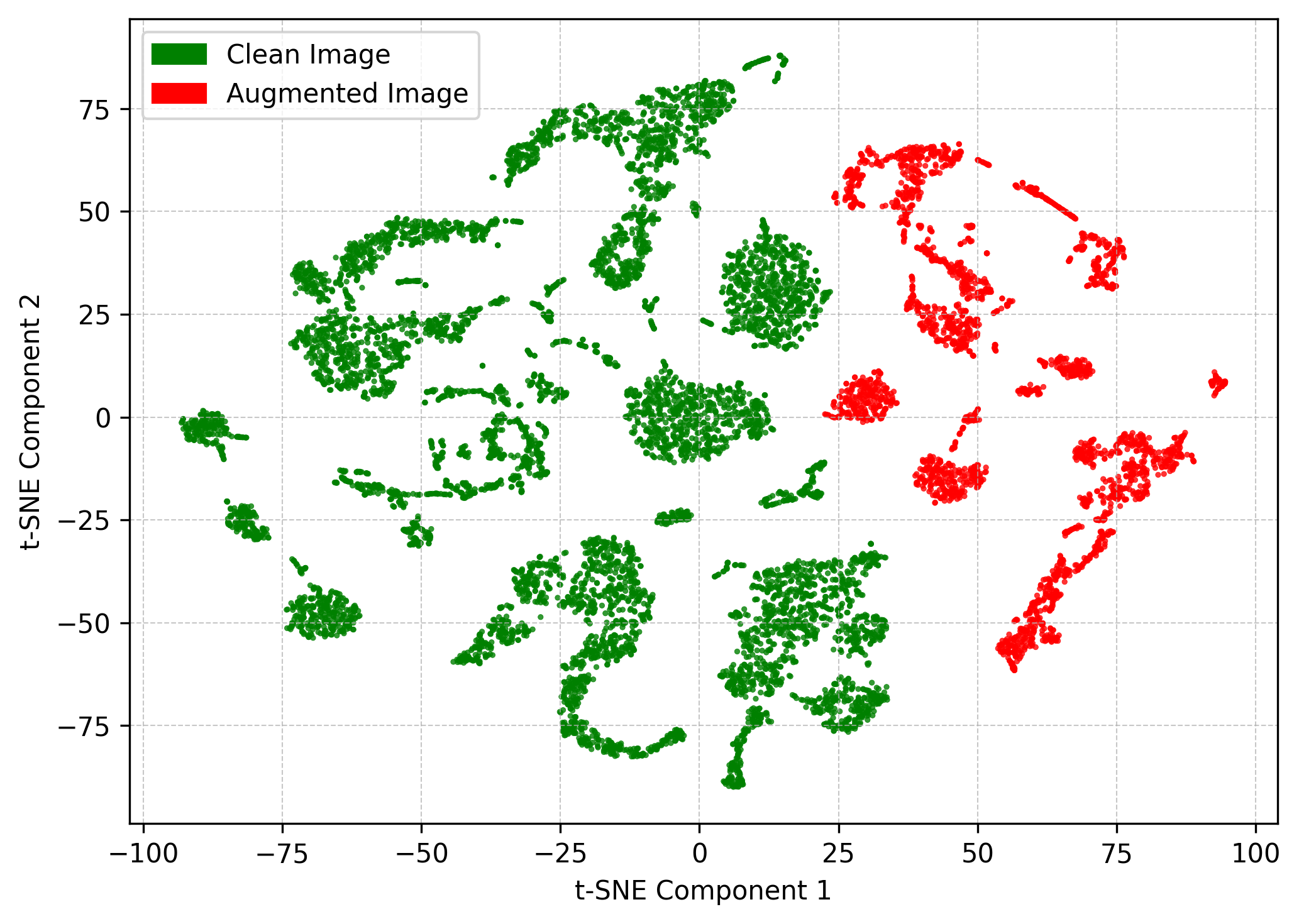}
    \caption{t-SNE plot of an 8D latent space representation generated for the clean ($X$) and motion blur augmented image dataset ($\bar{X}$)}
    \label{fig:aug_latent_space}
\end{figure}

By verifying the neural network controller over $\bar{C}_i$, we assess its robustness to input perturbations. SEVIN unifies formal verification and robustness verification within a single framework, leveraging generative AI techniques to scale the verification problem to the continuous domain.

\subsection{Designing Symbolic Formal Specifications}\label{sec:symbolic_properties}

The combined network $\mathcal{H}(\mathbf{z})$ undergoes verification against a set of \textbf{SAFETY} and \textbf{PERFORMANCE} specifications within both general formal verification and robustness verification frameworks. These specifications adhere to the Verification of Neural Network Library (VNN-LIB) standard, which is widely recognized for neural network verification benchmarks \cite{guidotti2023vnnlib}. The VNN-LIB specification standard builds upon the Open Neural Network Exchange (ONNX) format for model description and the Satisfiability Modulo Theory Library (SMT-LIB) format for property specification, ensuring compatibility and interoperability across various verification tools and platforms. The VNN-LIB standard allows designers to specify bounds on each input and output parameter of the neural network under verification, providing a highly expressive framework for defining verification constraints. The specifications are meticulously crafted to align with the driving scenarios outlined in Section~\ref{sec:driving_scenarios} 

\textbf{SAFETY} specifications are designed to guarantee that the neural network controller does not produce unsafe action values within a defined input polytope $C_i$. For instance, in natural language, a \textbf{SAFETY} specification for the  driving scenario might state that the neural network controller "always predicts a \emph{RIGHT} turn if the input image indicates a \emph{RIGHT} turn". Formally, this can be expressed as:
\begin{displaymath} \label{eq:safety}
    \varphi_\text{SAFETY}\coloneq \ \forall \mathbf{z} \in C_\text{right}, \{a_{\text{min}},a_{\text{max}}\} \in A_{\text{right}}
\end{displaymath}

where $\{a_{\text{min}},a_{\text{max}}\}$ can be calculated by solving the optimization problem from Equation~\eqref{eq:formal_obj_func}. $C_\text{right}$ represents the $d_z$-dimensional convex polytope in the latent space $Z$ corresponding to images indicative of right turns. The specification denotes for all latent variables $\mathbf{z}$ within $C_\text{right}$, the network's output $\mathcal{H}(\mathbf{z})$ belongs exclusively to the set of steering actions $A_{\text{right}}$ associated with a \emph{RIGHT} turn.

Conversely, \textbf{PERFORMANCE} specifications aim to ensure that a bounded set of control actions can be achieved from an input convex polytope within the latent space corresponding to the specification. This facilitates the assessment of the neural network controller's performance across different regions of the input space. The process involves partitioning the entire range of control actions into distinct subsets $A_{i=1}^I \subseteq A$. For each action subset $A_i$, the corresponding convex polytope $C_i$ in the latent space is determined as described in Section~\ref{sec:conv_poly}. An exemplary \textbf{PERFORMANCE} specification is presented below:
\begin{displaymath} \label{eq:performance}
    \varphi_\text{PERFORMANCE}\coloneq \ \forall \mathbf{z} \in C_\text{[0.2,0.5]}, \ \{0.2 \leq [a_{\text{min}},a_{\text{max}}] \leq 0.5\}
\end{displaymath}

where $\{a_{\text{min}},a_{\text{max}}\}$ can be calculated by solving the optimization problem from Equation~\eqref{eq:formal_obj_func}. In this context, $[0.2, 0.5]$ delineates the range of steering action values within which the combined network $\mathcal{H}(\mathbf{z})$ is expected to predict a steering value for all $\mathbf{z} \in C_{[0.2, 0.5]}$. This formalization ensures that the controller operates within the desired performance bounds across specified input regions.

\begin{table*}[t]
\centering
\caption{Verification Results for \textit{Robust Formal} Verification using SEVIN}
\label{tab:results_robust}
\scalebox{0.7}{
\begin{tabular}{|c|c|c|c|c|c|c|c|c|c|c|c|c|}
\hline
\multirow{2}{*}{Augmentation} & \multirow{2}{*}{Specification} & \multirow{2}{*}{Formula} & \multicolumn{2}{c|}{\textbf{NvidiaNet} $\delta_1$ = [80-120]\%} & \multicolumn{2}{c|}{\textbf{NvidiaNet} $\delta_2$ = [60-140]\%} & \multicolumn{2}{c|}{\textbf{ResNet18} $\delta_1$ = [80-120]\%} & \multicolumn{2}{c|}{\textbf{ResNet18} $\delta_2$ = [60-140]\%} & \multicolumn{2}{c|}{\textbf{ResNet18} $\delta_3$ = [50-150]\%} \\ \cline{4-13} 
                              &                                &                         & Result         & Time(s)       & Result         & Time(s)       & Result         & Time(s)        & Result         & Time(s)        & Result         & Time(s)        \\ \hline
\multirow{5}{*}{Brightness}   & \multirow{2}{*}{Safety}        & $\varphi_1$             & SAT            & 0.412         & SAT            & 0.4027        & SAT            & 0.6189         & SAT            & 0.6293         & SAT            & 0.7109         \\
                              &                                & $\varphi_2$             & SAT            & 0.3739        & SAT            & 0.3637        & SAT            & 0.565          & SAT            & 0.6001         & UNSAT          & -              \\ \cline{2-13} 
                              & \multirow{3}{*}{Performance}   & $\varphi_3$             & SAT            & 0.6075        & SAT           & 0.5804         & SAT            & 0.637          & SAT            & 0.6332         & SAT            & 0.6132         \\
                              &                                & $\varphi_4$             & SAT            & 0.326         & SAT            & 0.39          & SAT            & 0.5933         & SAT            & 0.5588         & SAT            & 0.5860         \\
                              &                                & $\varphi_5$             & UNSAT          & -             & UNSAT          & -             & SAT            & 0.853          & SAT            & 0.8011         & UNSAT          & -              \\ \hline
\multirow{2}{*}{} & \multirow{2}{*}{} & \multirow{2}{*}{} & \multicolumn{2}{c|}{\textbf{NvidiaNet} $\delta_1$ = \{1,2\}} & \multicolumn{2}{c|}{\textbf{NvidiaNet} $\delta_2$ = \{3,4\}} & \multicolumn{2}{c|}{\textbf{ResNet18} $\delta_1$ = \{1,2\}} & \multicolumn{2}{c|}{\textbf{ResNet18} $\delta_2$ = \{3,4\}} & \multicolumn{2}{c|}{\textbf{ResNet18} $\delta_3$ = \{5,6\}} \\ \cline{4-13} 
\multirow{5}{*}{Motion Blur}  & \multirow{2}{*}{Safety}        & $\varphi_1$             & SAT            & 0.304         & SAT            & 0.33          & SAT            & 0.6134         & SAT            & 0.6236         & SAT            & 0.7109         \\
                              &                                & $\varphi_2$             & SAT            & 0.3768        & SAT            & 0.367         & SAT            & 0.629          & SAT            & 0.627          & SAT            & 0.6193         \\ \cline{2-13} 
                              & \multirow{3}{*}{Performance}   & $\varphi_3$             & SAT            & 0.5731         & SAT          & 0.5043         & SAT            & 0.616          & SAT            & 0.6477        & SAT          & 0.8881         \\
                              &                                & $\varphi_4$             & SAT            & 0.3745        & SAT            & 0.4403        & SAT            & 0.6776         & SAT            & 0.6702         & SAT          & 0.7856         \\
                              &                                & $\varphi_5$             & SAT            & 0.5621        & UNSAT          & -             & SAT            & 0.822          & SAT            & 0.555          & SAT          & 1.816              \\ \hline
\end{tabular}
}
\end{table*}

\section{Experiments}\label{sec:experiments}

To test our proposed approach on an image based neural network controller, we choose an autonomous driving scenario where an AV drives itself around a track in a simulator. One of the goals of our experiments is to test and see if we can generate an interpretable latent representation of the image dataset collected by the front camera images. Once we do that, we want to make sure that we can configure the convex polytopes in the latent space as inputs to the verification problem. Finally, we aim to evaluate the controller's performance by conducting both \textit{vanilla formal} and \textit{robust formal} verifications, and compare the performance metrics of our approach to a general image-based neural network robustness verification problem (see more in Section~\ref{sec:scalability_comparison}). The VAE's and neural network controllers are trained on 2x NVIDIA A100 GPU's and the formal verification is carried out on a NVIDIA RTX 3090 GPU with 24GB of VRAM.

\subsection{Driving Scenarios}\label{sec:driving_scenarios}
The driving simulator collects RGB images ($\mathbf{x}$) from the AV's front-facing camera along with the steering control actions ($\mathbf{a}$). The AV drives on a custom, single-lane track created in {\fontfamily{cmtt}\selectfont RoadRunner} by {\fontfamily{cmtt}\selectfont MathWorks} and simulated in the {\fontfamily{cmtt}\selectfont CARLA} environment. The simulator's autopilot mode autonomously drives the vehicle, collecting control data, including the steering angle, as the control action ($\mathbf{a}$). The RGB images ($\mathbf{x}$) are resized to 80x64 grayscale images to reduce dimensionality while retaining essential lane information. A snapshot of the front camera view and the processed images used for training can be seen in Figure~\ref{fig:img_augmentations}.

We ensure that the camera captures features relevant to the \textbf{SAFETY} and \textbf{PERFORMANCE} properties discussed in Section~\ref{sec:symbolic_properties}. The images are pre-processed to retain only essential lane marking features, which reduces learning redundant features by the VAE, allowing for an easily distinguishable latent space based on differing action sets $A_i$.

\begin{table}[H]
\centering
\caption{Hyperparameters for VAE Training}
\scalebox{0.8}{
\begin{tabular}{|l|c|}
\hline
\textbf{Parameter}                 & \textbf{Value}        \\ \hline
Latent Dimension                   & 8                    \\ \hline
Optimizer                          & Adam                  \\ \hline
Learning Rate                      & \(1 \times 10^{-5}\)  \\ \hline
Weight Decay                       & \(1 \times 10^{-4}\) \\ \hline
Epochs                             & 20                   \\ \hline
$\beta$                            & 0.01                 \\ \hline
\end{tabular}
}
\end{table}

\subsection{Network Architectures}\label{sec:nnc_arch}
The GM-VAE used in our approach consists of an encoder and a decoder network, with convolutional and transposed convolutional layers, respectively, to process and reconstruct data. The encoder progressively increases channel sizes, while the decoder reduces them in reverse order. Each layer integrates batch normalization, ReLU activations, and dropout to prevent overfitting. The encoder begins with a linear layer, followed by three convolutional layers with increasing channel sizes: from 1 to 64, then 128, 256, and finally 512 channels. We employ \( K = 16 \) Gaussians in the mixture model to enhance the expressiveness of the latent representation. A Sigmoid activation function is applied at the output layer to ensure that the generated pixel values are within the range \([0, 1]\).

\begin{table}[H]
\centering
\caption{Results for \textit{Vanilla Formal} Verification using SEVIN}
\label{tab:results_vanilla}
\scalebox{0.8}{
\begin{tabular}{|c|c|c|c|c|c|}
\hline
\multicolumn{2}{|c|}{\multirow{3}{*}{Specification}} & \multicolumn{4}{c|}{NNC Architecture} \\ \cline{3-6}
\multicolumn{2}{|c|}{}                               & \multicolumn{2}{c|}{\textbf{NvidiaNet}} & \multicolumn{2}{c|}{\textbf{ResNet18}} \\ \cline{3-6}
\multicolumn{2}{|c|}{}                               & Results            & Time(s)          & Results            & Time(s)          \\ \hline
\multirow{2}{*}{Safety}      & $\varphi_1$          & SAT                & 0.417            & SAT                & 0.6343           \\ \cline{2-6}
                             & $\varphi_2$          & SAT                & 0.4430           & SAT                & 0.7023           \\ \hline
\multirow{3}{*}{Performance} & $\varphi_3$          & SAT                & 0.412            & SAT                & 0.6799           \\ \cline{2-6}
                             & $\varphi_4$          & SAT                & 0.3667           & SAT                & 0.5746           \\ \cline{2-6}
                             & $\varphi_5$          & SAT                & 0.6383           & SAT                & 0.8215           \\ \hline
\end{tabular}}
\end{table}

\subsection{Results}\label{results}

We evaluated our proposed method using the neural network verification tool \texttt{$\alpha-\beta-CROWN$}, offering certified bounds on model outputs under specified perturbations. This tool is suitable for verifying the safety and reliability of neural network controllers in autonomous systems. To assess both \textit{vanilla formal} and \textit{robust formal} verification methods (Section~\ref{sec:aug_latent_space}), we employed two \textbf{SAFETY} specifications and three \textbf{PERFORMANCE} specifications as described in Section~\ref{sec:symbolic_properties}.

\subsubsection{Image Augmentations}\label{sec:image_augmenations}

For the \textit{robust formal verifications}, we applied different levels ($\delta_{1,2,3}$) of augmentations to the image dataset to generate $\bar{X}$ for training the VAEs. The types and quantifications of the image augmentations are as follows:

\begin{itemize}
    \item \textbf{Brightness}: Datasets were generated by randomly varying image brightness levels within specified ranges $\delta_i$:
    \begin{itemize}
        \item $\delta_1$: 80\% to 120\% of original brightness.
        \item $\delta_2$: 60\% to 140\% of original brightness.
        \item $\delta_3$: 50\% to 150\% of original brightness.
    \end{itemize}
    \item \textbf{Vertical Motion Blur}: Datasets were generated by varying the degree of vertical motion blur kernels within the ranges:
    \begin{itemize}
        \item $\delta_1$: Kernel sizes of 1 and 2 pixels.
        \item $\delta_2$: Kernel sizes of 3 and 4 pixels.
        \item $\delta_3$: Kernel sizes of 5 and 6 pixels.
    \end{itemize}
\end{itemize}

\subsubsection{Specifications}

The \textbf{SAFETY} specifications used to verify the controller are defined as:
\begin{equation} \label{eq:exp_safety}
    \begin{aligned}
        &\varphi^1_\text{SAFETY} \coloneq \forall \mathbf{z} \in C_{[-ve]}, \quad \mathcal{H}(\mathbf{z}) \leq 0.0, \\
        &\varphi^2_\text{SAFETY} \coloneq \forall \mathbf{z} \in C_{[+ve]}, \quad \mathcal{H}(\mathbf{z}) \geq 0.0,
    \end{aligned}
\end{equation}

where $C_{[-ve]}$ and $C_{[+ve]}$ denote the latent space regions corresponding to negative and positive control actions, respectively.

The \textbf{PERFORMANCE} specifications are defined as:
\begin{displaymath} \label{eq:exp_perform}
    \begin{aligned}
        &\varphi^1_\text{PERFORM} \coloneq \forall \mathbf{z} \in C_{[-0.4,-0.1]}, \quad -0.4 \leq \mathcal{H}(\mathbf{z}) \leq -0.1, \\
        &\varphi^2_\text{PERFORM} \coloneq \forall \mathbf{z} \in C_{[-0.1,0.1]}, \quad -0.1 \leq \mathcal{H}(\mathbf{z}) \leq 0.1, \\
        &\varphi^3_\text{PERFORM} \coloneq  \forall \mathbf{z} \in C_{[0.1,0.4]}, \quad 0.1 \leq \mathcal{H}(\mathbf{z}) \leq 0.4,
    \end{aligned}
\end{displaymath}

where $C_{[a,b]}$ represents the latent space regions corresponding to control actions between $a$ and $b$.

These specifications were used for both \textit{vanilla formal} and \textit{robust formal} verification methods. The verification results, along with the input formal specifications, applied image augmentations, and evaluated neural network controllers, are summarized in Table~\ref{tab:results_vanilla} and Table~\ref{tab:results_robust}.

\subsubsection{Vanilla Formal Verification:}\label{sec:vani_formal_ver_results}
From the results, we can infer that both the vanilla and robust formal verification processes take < 1 second to conduct verification. This is due to the reduction in computational complexity of the processes as we introduce lower dimensional input spaces in the form of convex polytopes ($C$). Additionally, we also note that for the vanilla formal verification, the neural network controllers satisfy all of the specifications provided. This indicates that the controllers provide formal guarantees with respect to both the \textbf{SAFETY} and \textbf{PERFORMANCE} specifications when the input space $C_i$ belongs to clean image sets $X_i$. 

\subsubsection{Robust Formal Verification:}\label{sec:rob_formal_ver_results}
For the \textit{robust formal} verification process, we can notice specification $\varphi_5$ to be unsatisfactory for the NvidiaNet architecture, for both levels and types of augmentations, $\delta_1$ and $\delta_2$. The NvidiaNet architecture thus seems to be more susceptible to image augmentations and fairs poorly during the robustness verification process. In contrast, the ResNet18 architecture tends to perform fairly better than NvidiaNet for the $\delta_1$ and $\delta_2$ levels of augmentations for both \textbf{vertical motion blur} and \textbf{brightness} variation. It starts providing unsatisfactory verification results once the $\delta_3$ level of augmentations are applied to both types of image augmentations. This shows that the ResNet18 architecture is more robust than NvidiaNet in handling image perturbations for the case of autonomous driving scenarios.

\subsubsection{Scalability Comparison:}\label{sec:scalability_comparison}
We conduct general robustness verification on the neural network controller \( F(\mathbf{x}) \) using the $\alpha-\beta-CROWN$ toolbox and compare these results with those obtained via the SEVIN framework. We employ the same set of \textbf{SAFETY} specifications as described in Equation~\eqref{eq:exp_safety}, and we separate the brightness-augmented dataset \( \bar{X} \) into two subsets, \( \bar{X}_{[-\text{ve}]} \) and \( \bar{X}_{[+\text{ve}]} \), based on their corresponding negative and positive control action sets, \( A_{[-\text{ve}]} \) and \( A_{[+\text{ve}]} \). By converting \( F(\mathbf{x}) \) into the ONNX format and defining perturbation bounds for each dimension (i.e., pixel) of \( \mathbf{x} \), we leverage the neural network verification tool to perform robustness verification directly on the original high-dimensional image inputs. The upper and lower bounds are calculated for all images \( \mathbf{x} \in \bar{X}_{[-\text{ve}]} \) and \( \mathbf{x} \in \bar{X}_{[+\text{ve}]} \). The results are shown in Table~\ref{tab:results_sc}.\\
When comparing the results between Table~\ref{tab:results_robust} and \ref{tab:results_sc}, we observe that all the specifications are \textbf{SAT} for both methods. However, a significant difference lies in the time taken by the general method, which is almost ten times longer than that of the SEVIN method. This substantial difference is due to the fact that the input dimensionality of the verification problem using the SEVIN framework is approximately 600 times smaller than that of the general framework, making the problem computationally less complex and, consequently, more scalable to larger networks. Although for SEVIN, the combined network ($\mathcal{H}$) to be verified has many more layers than the neural network controller ($F$).

\begin{table}[H]
\centering
\caption{Results for General \textit{Robust Formal} Verification using the Neural Network Verification toolbox}
\label{tab:results_sc}
\scalebox{0.8}{
\begin{tabular}{|c|c|c|c|c|c|c|}
\hline
\multirow{3}{*}{\textbf{Augmentation}} & \multicolumn{2}{|c|}{\multirow{3}{*}{\textbf{Specification}}} & \multicolumn{4}{c|}{\textbf{NNC Architecture}} \\ \cline{4-7}
                                       & \multicolumn{2}{|c|}{}                                        & \multicolumn{2}{c|}{\textbf{NvidiaNet}}        & \multicolumn{2}{c|}{\textbf{ResNet18}}        \\ \cline{4-7}
                                       & \multicolumn{2}{|c|}{}                                        & \textbf{Results}      & \textbf{Time(s)}     & \textbf{Results}     & \textbf{Time(s)}     \\ \hline
\multirow{2}{*}{\textbf{Brightness}}   & \multirow{2}{*}{\textbf{Safety}}     & $\varphi_1$             & SAT                  & 3.213                & SAT                 & 4.257                \\ \cline{3-7}
                                       &                                      & $\varphi_2$             & SAT                  & 4.165                & SAT                 & 4.958                \\ \hline
\end{tabular}}
\end{table}

\section{Conclusion}
We provide a framework for scalable and interpretable formal verification of image based neural network controllers (SEVIN). Our approach involves developing a trained latent space representation for the image dataset used by the neural network controller. By using control action values as labels, we classify the latent variables and generate convex polytopes as input spaces for the verification process. We concatenate the decoder network of the Variational Autoencoder (VAE) with the neural network controller, allowing us to directly map lower-dimensional latent variables to higher-dimensional control action values. Once the verification input space is made interpretable, we construct specifications using symbolic languages such as Linear Temporal Logic (LTL). We then provide the neural network verification tool with the combined network and the specifications for verification. To enhance the formal robustness verification process, we generate an augmented dataset and retrain the VAEs to produce new latent representations.Finally, we test the SEVIN framework on two neural network controllers in an autonomous driving scenario, using two sets of specifications—\textbf{SAFETY} and \textbf{PERFORMANCE}—for both the standard formal verification and robustness verification processes. In the future, we aim to conduct reachability analysis using the SEVIN framework for neural network controllers and the plant model they control. We also aim to improve the performance of the decoder network to minimize reconstruction loss by using improved neural architectures.

\bibliographystyle{ACM-Reference-Format}
\bibliography{references}


\begin{thebibliography}{39}


\ifx \showCODEN    \undefined \def \showCODEN     #1{\unskip}     \fi
\ifx \showDOI      \undefined \def \showDOI       #1{#1}\fi
\ifx \showISBNx    \undefined \def \showISBNx     #1{\unskip}     \fi
\ifx \showISBNxiii \undefined \def \showISBNxiii  #1{\unskip}     \fi
\ifx \showISSN     \undefined \def \showISSN      #1{\unskip}     \fi
\ifx \showLCCN     \undefined \def \showLCCN      #1{\unskip}     \fi
\ifx \shownote     \undefined \def \shownote      #1{#1}          \fi
\ifx \showarticletitle \undefined \def \showarticletitle #1{#1}   \fi
\ifx \showURL      \undefined \def \showURL       {\relax}        \fi
\providecommand\bibfield[2]{#2}
\providecommand\bibinfo[2]{#2}
\providecommand\natexlab[1]{#1}
\providecommand\showeprint[2][]{arXiv:#2}

\bibitem[Akazaki and Liu(2018)]%
        {akazaki2018falsification}
\bibfield{author}{\bibinfo{person}{Takumi Akazaki} {and} \bibinfo{person}{Yang Liu}.} \bibinfo{year}{2018}\natexlab{}.
\newblock \showarticletitle{Falsification of Cyber-Physical Systems Using Deep Reinforcement Learning}. In \bibinfo{booktitle}{\emph{International Symposium on Formal Methods}}. Springer, \bibinfo{pages}{456--465}.
\newblock


\bibitem[Al-Nuaimi et~al\mbox{.}(2021)]%
        {Al-Nuaimi2021Hybrid}
\bibfield{author}{\bibinfo{person}{Mohammed Al-Nuaimi}, \bibinfo{person}{Sapto Wibowo}, \bibinfo{person}{Hongyang Qu}, \bibinfo{person}{Jonathan Aitken}, {and} \bibinfo{person}{Sandor Veres}.} \bibinfo{year}{2021}\natexlab{}.
\newblock \showarticletitle{Hybrid Verification Technique for Decision-Making of Self-Driving Vehicles}.
\newblock \bibinfo{journal}{\emph{Journal of Sensor and Actuator Networks}} \bibinfo{volume}{10}, \bibinfo{number}{3} (\bibinfo{year}{2021}), \bibinfo{pages}{42}.
\newblock
\urldef\tempurl%
\url{https://doi.org/10.3390/jsan10030042}
\showDOI{\tempurl}


\bibitem[Barber et~al\mbox{.}(1996)]%
        {Barber1996Quickhull}
\bibfield{author}{\bibinfo{person}{C~Bradford Barber}, \bibinfo{person}{David~P Dobkin}, {and} \bibinfo{person}{Hannu Huhdanpaa}.} \bibinfo{year}{1996}\natexlab{}.
\newblock \showarticletitle{The Quickhull Algorithm for Convex Hulls}.
\newblock \bibinfo{journal}{\emph{ACM Transactions on Mathematical Software (TOMS)}} (\bibinfo{year}{1996}).
\newblock


\bibitem[Bronstein et~al\mbox{.}(2017)]%
        {bronstein2017geometric}
\bibfield{author}{\bibinfo{person}{Michael~M Bronstein}, \bibinfo{person}{Joan Bruna}, \bibinfo{person}{Yann LeCun}, \bibinfo{person}{Arthur Szlam}, {and} \bibinfo{person}{Pierre Vandergheynst}.} \bibinfo{year}{2017}\natexlab{}.
\newblock \showarticletitle{Geometric Deep Learning: Going Beyond Euclidean Data}.
\newblock \bibinfo{journal}{\emph{IEEE Signal Processing Magazine}} (\bibinfo{year}{2017}).
\newblock


\bibitem[Bunel et~al\mbox{.}(2018)]%
        {bunel2018unified}
\bibfield{author}{\bibinfo{person}{Rudy Bunel}, \bibinfo{person}{Isil~Dillig Turkaslan}, {and} \bibinfo{person}{Philip~HS Torr}.} \bibinfo{year}{2018}\natexlab{}.
\newblock \showarticletitle{A Unified View of Piecewise Linear Neural Network Verification}.
\newblock \bibinfo{journal}{\emph{Advances in Neural Information Processing Systems}} (\bibinfo{year}{2018}).
\newblock


\bibitem[Choromanska et~al\mbox{.}(2015)]%
        {choromanska2015loss}
\bibfield{author}{\bibinfo{person}{Anna Choromanska}, \bibinfo{person}{Mikael Henaff}, \bibinfo{person}{Michael Mathieu}, \bibinfo{person}{Gerard Ben~Arous}, {and} \bibinfo{person}{Yann LeCun}.} \bibinfo{year}{2015}\natexlab{}.
\newblock \showarticletitle{The Loss Surfaces of Multilayer Networks}.
\newblock \bibinfo{journal}{\emph{Proceedings of the Eighteenth International Conference on Artificial Intelligence and Statistics}} (\bibinfo{year}{2015}).
\newblock


\bibitem[Clarke(1998)]%
        {clarke1998nonsmooth}
\bibfield{author}{\bibinfo{person}{Frank~H Clarke}.} \bibinfo{year}{1998}\natexlab{}.
\newblock \bibinfo{booktitle}{\emph{Nonsmooth Analysis and Control Theory}}.
\newblock \bibinfo{publisher}{Springer}.
\newblock


\bibitem[Dilokthanakul et~al\mbox{.}(2016)]%
        {dilokthanakul2016deep}
\bibfield{author}{\bibinfo{person}{Nat Dilokthanakul}, \bibinfo{person}{Pedro~AM Mediano}, \bibinfo{person}{Marta Garnelo}, \bibinfo{person}{Matthew~CH Lee}, \bibinfo{person}{Hugh Salimbeni}, \bibinfo{person}{Kai Arulkumaran}, {and} \bibinfo{person}{Murray Shanahan}.} \bibinfo{year}{2016}\natexlab{}.
\newblock \showarticletitle{Deep Unsupervised Clustering with Gaussian Mixture Variational Autoencoders}.
\newblock \bibinfo{journal}{\emph{arXiv preprint arXiv:1611.02648}} (\bibinfo{year}{2016}).
\newblock


\bibitem[Ehlers(2017)]%
        {ehlers2017formal}
\bibfield{author}{\bibinfo{person}{R{\"u}diger Ehlers}.} \bibinfo{year}{2017}\natexlab{}.
\newblock \showarticletitle{Formal Verification of Piecewise Linear Feed-Forward Neural Networks}.
\newblock \bibinfo{journal}{\emph{arXiv preprint arXiv:1705.01320}} (\bibinfo{year}{2017}).
\newblock


\bibitem[Gehr et~al\mbox{.}(2018)]%
        {Gehr2018AI2:Networks}
\bibfield{author}{\bibinfo{person}{Timon Gehr}, \bibinfo{person}{Matthew Mirman}, \bibinfo{person}{Dana Drachsler-Cohen}, \bibinfo{person}{Petar Tsankov}, {and} \bibinfo{person}{Martin Vechev}.} \bibinfo{year}{2018}\natexlab{}.
\newblock \showarticletitle{AI2: Safety and Robustness Certification of Neural Networks with Abstract Interpretation}.
\newblock \bibinfo{journal}{\emph{2018 IEEE Symposium on Security and Privacy (SP)}} (\bibinfo{year}{2018}).
\newblock


\bibitem[Gr{\"u}nbaum(2003)]%
        {grunbaum2003convex}
\bibfield{author}{\bibinfo{person}{Branko Gr{\"u}nbaum}.} \bibinfo{year}{2003}\natexlab{}.
\newblock \bibinfo{booktitle}{\emph{Convex Polytopes}}. Vol.~\bibinfo{volume}{221}.
\newblock \bibinfo{publisher}{Springer Science \& Business Media}.
\newblock


\bibitem[Guidotti et~al\mbox{.}(2023)]%
        {guidotti2023vnnlib}
\bibfield{author}{\bibinfo{person}{Dario Guidotti}, \bibinfo{person}{Stefano Demarchi}, \bibinfo{person}{Armando Tacchella}, {and} \bibinfo{person}{Luca Pulina}.} \bibinfo{year}{2023}\natexlab{}.
\newblock \bibinfo{title}{The Verification of Neural Networks Library (VNN-LIB)}.
\newblock
\newblock
\urldef\tempurl%
\url{https://www.vnnlib.org}
\showURL{%
\tempurl}
\newblock
\shownote{Accessed: 2023}.


\bibitem[Higgins et~al\mbox{.}(2017)]%
        {higgins2017betaVAE}
\bibfield{author}{\bibinfo{person}{Irina Higgins}, \bibinfo{person}{Loic Matthey}, \bibinfo{person}{Arka Pal}, \bibinfo{person}{Christopher Burgess}, \bibinfo{person}{Alexander Glorot-Xavier}, {and} \bibinfo{person}{Matthew Botvinick}.} \bibinfo{year}{2017}\natexlab{}.
\newblock \showarticletitle{beta-VAE: Learning Basic Visual Concepts with a Constrained Variational Framework}.
\newblock \bibinfo{journal}{\emph{International Conference on Learning Representations}} (\bibinfo{year}{2017}).
\newblock


\bibitem[Huang et~al\mbox{.}(2017)]%
        {huang2017safety}
\bibfield{author}{\bibinfo{person}{Xiaowei Huang}, \bibinfo{person}{Marta Kwiatkowska}, \bibinfo{person}{Sen Wang}, {and} \bibinfo{person}{Min Wu}.} \bibinfo{year}{2017}\natexlab{}.
\newblock \showarticletitle{Safety Verification of Deep Neural Networks}.
\newblock \bibinfo{journal}{\emph{arXiv preprint arXiv:1610.06940}} (\bibinfo{year}{2017}).
\newblock


\bibitem[Julian et~al\mbox{.}(2020)]%
        {Julian2020Validation}
\bibfield{author}{\bibinfo{person}{Kyle~D. Julian}, \bibinfo{person}{Ritchie Lee}, {and} \bibinfo{person}{Mykel~J. Kochenderfer}.} \bibinfo{year}{2020}\natexlab{}.
\newblock \showarticletitle{Validation of Image-Based Neural Network Controllers through Adaptive Stress Testing}.
\newblock \bibinfo{journal}{\emph{IEEE Transactions on Intelligent Transportation Systems}} \bibinfo{volume}{22}, \bibinfo{number}{5} (\bibinfo{year}{2020}), \bibinfo{pages}{2862--2871}.
\newblock
\urldef\tempurl%
\url{https://doi.org/10.1109/TITS.2020.2988147}
\showDOI{\tempurl}


\bibitem[Kaiser et~al\mbox{.}(2021)]%
        {kaiser2021smt}
\bibfield{author}{\bibinfo{person}{Eli Kaiser}, \bibinfo{person}{Saif Ahmed}, {and} \bibinfo{person}{Tommaso Dreossi}.} \bibinfo{year}{2021}\natexlab{}.
\newblock \showarticletitle{SMT-Based Verification of Neural Network Controllers for Autonomous Systems}.
\newblock \bibinfo{journal}{\emph{IEEE Transactions on Cybernetics}} (\bibinfo{year}{2021}).
\newblock


\bibitem[Katz et~al\mbox{.}(2017a)]%
        {Katz2017Reluplex:Networks}
\bibfield{author}{\bibinfo{person}{Guy Katz}, \bibinfo{person}{Clark Barrett}, \bibinfo{person}{David~L Dill}, \bibinfo{person}{Kyle Julian}, {and} \bibinfo{person}{Mykel~J Kochenderfer}.} \bibinfo{year}{2017}\natexlab{a}.
\newblock \showarticletitle{Reluplex: An Efficient SMT Solver for Verifying Deep Neural Networks}.
\newblock \bibinfo{journal}{\emph{arXiv preprint arXiv:1702.01135}} (\bibinfo{year}{2017}).
\newblock


\bibitem[Katz et~al\mbox{.}(2017b)]%
        {KatzVerificationModels}
\bibfield{author}{\bibinfo{person}{Guy Katz}, \bibinfo{person}{Clark Barrett}, \bibinfo{person}{David~L Dill}, \bibinfo{person}{Kyle Julian}, {and} \bibinfo{person}{Mykel~J Kochenderfer}.} \bibinfo{year}{2017}\natexlab{b}.
\newblock \showarticletitle{Towards Scalable Verification for All Neural Networks}.
\newblock \bibinfo{journal}{\emph{arXiv preprint arXiv:1702.01135}} (\bibinfo{year}{2017}).
\newblock


\bibitem[Katz et~al\mbox{.}(2021)]%
        {Katz2021Verification}
\bibfield{author}{\bibinfo{person}{Sydney~M. Katz}, \bibinfo{person}{Anthony~L. Corso}, \bibinfo{person}{Christopher~A. Strong}, {and} \bibinfo{person}{Mykel~J. Kochenderfer}.} \bibinfo{year}{2021}\natexlab{}.
\newblock \showarticletitle{Verification of Image-based Neural Network Controllers Using Generative Models}.
\newblock \bibinfo{journal}{\emph{IEEE Transactions on Neural Networks and Learning Systems}} \bibinfo{volume}{33}, \bibinfo{number}{7} (\bibinfo{year}{2021}), \bibinfo{pages}{3106--3120}.
\newblock
\urldef\tempurl%
\url{https://doi.org/10.1109/TNNLS.2021.3087057}
\showDOI{\tempurl}


\bibitem[Kawaguchi(2016)]%
        {kawaguchi2016deep}
\bibfield{author}{\bibinfo{person}{Kenji Kawaguchi}.} \bibinfo{year}{2016}\natexlab{}.
\newblock \showarticletitle{Deep Learning without Poor Local Minima}.
\newblock \bibinfo{journal}{\emph{arXiv preprint arXiv:1605.07110}} (\bibinfo{year}{2016}).
\newblock


\bibitem[Kingma and Welling(2013)]%
        {kingma2013autoencoding}
\bibfield{author}{\bibinfo{person}{Diederik~P Kingma} {and} \bibinfo{person}{Max Welling}.} \bibinfo{year}{2013}\natexlab{}.
\newblock \showarticletitle{Auto-Encoding Variational Bayes}.
\newblock \bibinfo{journal}{\emph{arXiv preprint arXiv:1312.6114}} (\bibinfo{year}{2013}).
\newblock


\bibitem[Liu et~al\mbox{.}(2019)]%
        {liu2019algorithms}
\bibfield{author}{\bibinfo{person}{Changliu Liu}, \bibinfo{person}{Tamar Arnon}, \bibinfo{person}{Christopher Lazarus}, \bibinfo{person}{Alexander Strong}, \bibinfo{person}{Clark Barrett}, {and} \bibinfo{person}{Mykel~J Kochenderfer}.} \bibinfo{year}{2019}\natexlab{}.
\newblock \showarticletitle{Algorithms for Verifying Neural Networks}.
\newblock \bibinfo{journal}{\emph{arXiv preprint arXiv:1903.06758}} (\bibinfo{year}{2019}).
\newblock


\bibitem[Montufar et~al\mbox{.}(2014)]%
        {montufar2014number}
\bibfield{author}{\bibinfo{person}{Guido~F Montufar}, \bibinfo{person}{Razvan Pascanu}, \bibinfo{person}{Kyunghyun Cho}, {and} \bibinfo{person}{Yoshua Bengio}.} \bibinfo{year}{2014}\natexlab{}.
\newblock \showarticletitle{On the Number of Linear Regions of Deep Neural Networks}.
\newblock \bibinfo{journal}{\emph{arXiv preprint arXiv:1402.1869}} (\bibinfo{year}{2014}).
\newblock


\bibitem[Pnueli(1977)]%
        {pnueli1977temporal}
\bibfield{author}{\bibinfo{person}{Amir Pnueli}.} \bibinfo{year}{1977}\natexlab{}.
\newblock \showarticletitle{The Temporal Logic of Programs}.
\newblock \bibinfo{journal}{\emph{18th Annual Symposium on Foundations of Computer Science (sfcs 1977)}} (\bibinfo{year}{1977}).
\newblock


\bibitem[Raghu et~al\mbox{.}(2017)]%
        {raghu2017expressive}
\bibfield{author}{\bibinfo{person}{Maithra Raghu}, \bibinfo{person}{Ben Poole}, \bibinfo{person}{Jon Kleinberg}, \bibinfo{person}{Surya Ganguli}, {and} \bibinfo{person}{Jascha Sohl-Dickstein}.} \bibinfo{year}{2017}\natexlab{}.
\newblock \showarticletitle{On the Expressive Power of Neural Networks with Relu Activations}.
\newblock \bibinfo{journal}{\emph{arXiv preprint arXiv:1711.02060}} (\bibinfo{year}{2017}).
\newblock


\bibitem[Rezende et~al\mbox{.}(2014)]%
        {pmlr-v32-rezende14}
\bibfield{author}{\bibinfo{person}{Danilo~Jimenez Rezende}, \bibinfo{person}{Shakir Mohamed}, {and} \bibinfo{person}{Daan Wierstra}.} \bibinfo{year}{2014}\natexlab{}.
\newblock \showarticletitle{Stochastic Backpropagation and Approximate Inference in Deep Generative Models}.
\newblock  (\bibinfo{year}{2014}).
\newblock


\bibitem[Ruan et~al\mbox{.}(2018)]%
        {ruan2018reachability}
\bibfield{author}{\bibinfo{person}{Wenjie Ruan}, \bibinfo{person}{Xinming Huang}, {and} \bibinfo{person}{Marta Kwiatkowska}.} \bibinfo{year}{2018}\natexlab{}.
\newblock \showarticletitle{Reachability Analysis of Deep Neural Networks with Provable Guarantees}.
\newblock \bibinfo{journal}{\emph{arXiv preprint arXiv:1805.02242}} (\bibinfo{year}{2018}).
\newblock


\bibitem[Schneider(1993)]%
        {schneider1993convex}
\bibfield{author}{\bibinfo{person}{Rolf Schneider}.} \bibinfo{year}{1993}\natexlab{}.
\newblock \showarticletitle{Convex Bodies: The Brunn-Minkowski Theory}.
\newblock \bibinfo{journal}{\emph{Cambridge University Press}} (\bibinfo{year}{1993}).
\newblock


\bibitem[Singh et~al\mbox{.}(2018)]%
        {singh2018fast}
\bibfield{author}{\bibinfo{person}{Gagandeep Singh}, \bibinfo{person}{Timon Gehr}, \bibinfo{person}{Markus P{\"u}schel}, {and} \bibinfo{person}{Martin Vechev}.} \bibinfo{year}{2018}\natexlab{}.
\newblock \showarticletitle{Fast and Effective Robustness Certification}.
\newblock \bibinfo{journal}{\emph{Advances in Neural Information Processing Systems}} (\bibinfo{year}{2018}).
\newblock


\bibitem[Singh et~al\mbox{.}(2019)]%
        {Singh2019AnAbstraction}
\bibfield{author}{\bibinfo{person}{Gagandeep Singh}, \bibinfo{person}{Timon Gehr}, \bibinfo{person}{Markus P{\"u}schel}, {and} \bibinfo{person}{Martin Vechev}.} \bibinfo{year}{2019}\natexlab{}.
\newblock \showarticletitle{An Abstraction-Based Framework for Neural Network Verification}.
\newblock \bibinfo{journal}{\emph{arXiv preprint arXiv:1810.09031}} (\bibinfo{year}{2019}).
\newblock


\bibitem[Tjeng et~al\mbox{.}(2019)]%
        {tjeng2019evaluating}
\bibfield{author}{\bibinfo{person}{Vincent Tjeng}, \bibinfo{person}{Kai Xiao}, {and} \bibinfo{person}{Russ Tedrake}.} \bibinfo{year}{2019}\natexlab{}.
\newblock \showarticletitle{Evaluating Robustness of Neural Networks with Mixed Integer Programming}.
\newblock \bibinfo{journal}{\emph{arXiv preprint arXiv:1711.07356}} (\bibinfo{year}{2019}).
\newblock


\bibitem[Tomczak and Welling(2018)]%
        {tomczak2018vae}
\bibfield{author}{\bibinfo{person}{Jakub~M Tomczak} {and} \bibinfo{person}{Max Welling}.} \bibinfo{year}{2018}\natexlab{}.
\newblock \showarticletitle{VAE with a VampPrior}.
\newblock \bibinfo{journal}{\emph{Proceedings of the 21st International Conference on Artificial Intelligence and Statistics}} (\bibinfo{year}{2018}).
\newblock


\bibitem[Tran et~al\mbox{.}(2020)]%
        {Tran2020Verification}
\bibfield{author}{\bibinfo{person}{Hoang~Dung Tran}, \bibinfo{person}{Stanley Bak}, \bibinfo{person}{Weiming Xiang}, {and} \bibinfo{person}{Taylor~T. Johnson}.} \bibinfo{year}{2020}\natexlab{}.
\newblock \showarticletitle{Verification of Deep Convolutional Neural Networks Using ImageStars}. In \bibinfo{booktitle}{\emph{Computer Aided Verification: 32nd International Conference, CAV 2020}}. \bibinfo{publisher}{Springer}, \bibinfo{pages}{18--42}.
\newblock
\urldef\tempurl%
\url{https://doi.org/10.1007/978-3-030-53288-8_2}
\showDOI{\tempurl}


\bibitem[van~der Maaten and Hinton(2008)]%
        {vanDerMaaten2008}
\bibfield{author}{\bibinfo{person}{Laurens van~der Maaten} {and} \bibinfo{person}{Geoffrey Hinton}.} \bibinfo{year}{2008}\natexlab{}.
\newblock \showarticletitle{Visualizing Data using t-SNE}.
\newblock \bibinfo{journal}{\emph{Journal of Machine Learning Research}} (\bibinfo{year}{2008}).
\newblock


\bibitem[Vasilache et~al\mbox{.}(2022)]%
        {vasilache2022verifying}
\bibfield{author}{\bibinfo{person}{Andrei Vasilache}, \bibinfo{person}{Lucas Brodbeck}, {and} \bibinfo{person}{Tommaso Dreossi}.} \bibinfo{year}{2022}\natexlab{}.
\newblock \showarticletitle{Verifying Temporal Logic Specifications in Neural Network Controllers}.
\newblock \bibinfo{journal}{\emph{2022 ACM/IEEE 6th International Conference on Formal Methods in Software Engineering (FormaliSE)}} (\bibinfo{year}{2022}).
\newblock


\bibitem[Wang et~al\mbox{.}(2021)]%
        {wang2021beta}
\bibfield{author}{\bibinfo{person}{Shiqi Wang}, \bibinfo{person}{Huan Zhang}, \bibinfo{person}{Kaidi Xu}, \bibinfo{person}{Xue Lin}, \bibinfo{person}{Suman Jana}, \bibinfo{person}{Cho-Jui Hsieh}, {and} \bibinfo{person}{J~Zico Kolter}.} \bibinfo{year}{2021}\natexlab{}.
\newblock \showarticletitle{{Beta-CROWN}: Efficient bound propagation with per-neuron split constraints for complete and incomplete neural network verification}.
\newblock \bibinfo{journal}{\emph{Advances in Neural Information Processing Systems}}  \bibinfo{volume}{34} (\bibinfo{year}{2021}).
\newblock


\bibitem[Xu et~al\mbox{.}(2020)]%
        {xu2020automatic}
\bibfield{author}{\bibinfo{person}{Kaidi Xu}, \bibinfo{person}{Zhouxing Shi}, \bibinfo{person}{Huan Zhang}, \bibinfo{person}{Yihan Wang}, \bibinfo{person}{Kai-Wei Chang}, \bibinfo{person}{Minlie Huang}, \bibinfo{person}{Bhavya Kailkhura}, \bibinfo{person}{Xue Lin}, {and} \bibinfo{person}{Cho-Jui Hsieh}.} \bibinfo{year}{2020}\natexlab{}.
\newblock \showarticletitle{Automatic perturbation analysis for scalable certified robustness and beyond}.
\newblock \bibinfo{journal}{\emph{Advances in Neural Information Processing Systems}}  \bibinfo{volume}{33} (\bibinfo{year}{2020}).
\newblock


\bibitem[Xu et~al\mbox{.}(2021)]%
        {xu2021fast}
\bibfield{author}{\bibinfo{person}{Kaidi Xu}, \bibinfo{person}{Huan Zhang}, \bibinfo{person}{Shiqi Wang}, \bibinfo{person}{Yihan Wang}, \bibinfo{person}{Suman Jana}, \bibinfo{person}{Xue Lin}, {and} \bibinfo{person}{Cho-Jui Hsieh}.} \bibinfo{year}{2021}\natexlab{}.
\newblock \showarticletitle{{Fast and Complete}: Enabling Complete Neural Network Verification with Rapid and Massively Parallel Incomplete Verifiers}. In \bibinfo{booktitle}{\emph{International Conference on Learning Representations}}.
\newblock
\urldef\tempurl%
\url{https://openreview.net/forum?id=nVZtXBI6LNn}
\showURL{%
\tempurl}


\bibitem[Zhang et~al\mbox{.}(2018)]%
        {zhang2018efficient}
\bibfield{author}{\bibinfo{person}{Huan Zhang}, \bibinfo{person}{Tsui-Wei Weng}, \bibinfo{person}{Pin-Yu Chen}, \bibinfo{person}{Cho-Jui Hsieh}, {and} \bibinfo{person}{Luca Daniel}.} \bibinfo{year}{2018}\natexlab{}.
\newblock \showarticletitle{Efficient Neural Network Robustness Certification with General Activation Functions}.
\newblock \bibinfo{journal}{\emph{Advances in Neural Information Processing Systems}}  \bibinfo{volume}{31} (\bibinfo{year}{2018}), \bibinfo{pages}{4939--4948}.
\newblock
\urldef\tempurl%
\url{https://arxiv.org/pdf/1811.00866.pdf}
\showURL{%
\tempurl}


\end{thebibliography}


\end{document}